\documentclass[sigconf]{acmart}

\AtBeginDocument{%
  \providecommand\BibTeX{{%
    \normalfont B\kern-0.5em{\scshape i\kern-0.25em b}\kern-0.8em\TeX}}}

\usepackage[utf8]{inputenc}
\usepackage{fontenc} 
\usepackage{url}            
 \usepackage{amsfonts}       
\usepackage{nicefrac}       
\usepackage{microtype}
\usepackage{enumitem}
\usepackage{isomath}
\usepackage{graphicx,subcaption}
\usepackage{grffile}
\usepackage{pdfpages} 

\usepackage{booktabs} 
\usepackage{mathtools,dsfont}

\usepackage{algorithm}
\usepackage{algpseudocode}
\usepackage{tikz}
\usetikzlibrary{bayesnet,arrows}

\usepackage[nameinlink,sort&compress]{cleveref}
\crefname{construction}{Construction}{Constructions}
\crefname{claim}{Claim}{Claims}
\crefname{paragraph}{Paragraph}{Paragraphs}
\crefname{observation}{Observation}{Observations}
\crefname{theorem}{Theorem}{Theorems}
\crefname{lemma}{Lemma}{Lemmata}
\crefname{proposition}{Proposition}{Propositions}
\crefname{corollary}{Corollary}{Corollaries}
\crefname{remark}{Remark}{Remarks}
\crefname{section}{Section}{sections}
\crefname{chapter}{Chapter}{Chapters}
\crefname{figure}{Figure}{Figures}
\crefname{table}{Table}{Tables}
\crefname{definition}{Definition}{Definitions}
\crefname{algorithm}{Algorithm}{Algorithms}
\crefname{equation}{Equation}{Equations}
\crefname{appendix}{Appendix}{Appendices}


\newcounter{parentnumber}
\newtheorem{lemma}{Lemma}

\newtheorem{example}{Example}
\newtheorem{remark}{Remark}

\theoremstyle{definition}
\newtheorem{definition}{Definition}

\usepackage{algpseudocode}
\tikzstyle{privplate} = [draw, rectangle, dashed, fit=#1]
\tikzstyle{privplate caption} = [caption, node distance=0, inner sep=0pt,
above left=10pt and 0pt of #1.north east]

\tikzstyle{decision} = [rectangle, text centered, minimum width=2em, minimum height=2em, draw=black]
\let\olddbar\|
\renewcommand{\|}{\,{\olddbar}\,}

\newcommand{\util}{U}

\newcommand{\param}{{\boldsymbol{\theta}}}

\newcommand{\Param}{\Theta}
\newcommand{\vparam}{{\boldsymbol{\theta}}}
\renewcommand{\Pr}{\mathbb{P}}

\newcommand{\E}{\mathbb{E}}

\newcommand{\CX}{\mathcal{X}}
\newcommand{\CY}{\mathcal{Y}}

\newcommand{\bx}{\mathbf{x}}
\newcommand{\by}{\mathbf{y}}
\newcommand{\bct}{\mathbf{b}}

\newcommand{\EMC}{\Delta}

\newcommand{\egpol}{\pol_{\text{egal}}}
\newcommand{\ut}{u}

\DeclareMathOperator*{\argmax}{arg\,max}

\newcommand \gender {\mathrm{gen}}
\newcommand \sk {\mathrm{sk}}
\newcommand \gpa {\mathrm{gpa}}
\newcommand \tal {\mathrm{tal}}

\newcommand \pol {\ensuremath{\pi}}
\newcommand \Pol {\ensuremath{\Pi}}

\newcommand \bel {\ensuremath{\beta}}

\newcommand \act {\mathbf{a}} 
\newcommand \Act {\ensuremath{\mathcal{A}}}

\newcommand \defn {\mathrel{\triangleq}}

\newcommand{\pop}{\mathcal{N}}
\newcommand{\nPop}{N}

\newcommand{\distr}{\xi}

\renewcommand{\colon}{\, : \,}
\newif\ifnotes
\newif\iflongversion
\newif\ifshortversion
\newif\ifpreprint

\newcommand\ind[1] {\mathbb{I}\left\lbrace #1 \right\rbrace}
\newcommand\cset[2] {\left\{#1 ~\middle|~ #2\right\}}

\newcommand\Reals{{\mathbb{R}}}

\newcommand{\devswap}{\text{Dev}_\text{swap}}
\newcommand{\devchange}{\text{Dev}_\text{local}}

\tikzstyle{place}=[circle,draw=black,inner sep=0mm, minimum size=6mm]

\tikzstyle{utility}=[diamond,draw=black,draw=blue!50,fill=blue!10,inner sep=0mm, minimum size=8mm]
\tikzstyle{select}=[rectangle,draw=black,draw=blue!50,fill=blue!10,inner sep=0mm, minimum size=6mm]
\tikzstyle{hidden}=[dashed,draw=black,fill=red!10]
\tikzstyle{RV}=[circle,draw=black,draw=blue!50,fill=blue!10,inner sep=0mm, minimum size=6mm]

\tikzstyle{transition}=[rectangle,draw=black!50,fill=black!20,thick]
\tikzstyle{someset}=[circle,draw=black,minimum size=8mm]
\tikzstyle{point}=[circle,draw=black,fill=black]

\newcommand{\cdcomment}[1]{[CD: \textcolor{red}{#1}]}

\DeclareMathOperator{\emc}{EMC}
\DeclareMathOperator{\shapley}{Shapley}

\PassOptionsToPackage{options}{natbib}

\copyrightyear{2022}
\acmYear{2022}
\setcopyright{acmlicensed}
\acmConference[EAAMO '22]{Equity and Access in Algorithms, Mechanisms, and Optimization}{October 6--9, 2022}{Arlington, VA, USA}
\acmBooktitle{Equity and Access in Algorithms, Mechanisms, and Optimization (EAAMO '22), October 6--9, 2022, Arlington, VA, USA}
\acmPrice{15.00}
\acmDOI{10.1145/3551624.3555305}
\acmISBN{978-1-4503-9477-2/22/10}

\begin{document}

\title{On Meritocracy in Optimal Set Selection}

\author{Thomas Kleine Buening}
\email{thomkl@ifi.uio.no}
\affiliation{
\institution{University of Oslo}
\streetaddress{Gaustadalléen 23B}
\city{Oslo}
\country{Norway}
}

\author{Meirav Segal}
\email{meiravs@ifi.uio.no}
\affiliation{%
  \institution{University of Oslo}
\streetaddress{Gaustadalléen 23B}
\city{Oslo}
\country{Norway}
}

\author{Debabrota Basu}
\email{debabrota.basu@inria.fr}
\affiliation{%
  \institution{Scool, Inria, Univ. Lille, Centrale Lille UMR 9189 – CRIStAL}
  \streetaddress{Parc Scientifique de la Haute-Borne}
  \city{Villeneuve-d'Ascq}
  \country{France}}

\author{Anne-Marie George}
\email{annemage@ifi.uio.no}
\affiliation{%
  \institution{University of Oslo}
\streetaddress{Gaustadalléen 23B}
\city{Oslo}
\country{Norway}
}

\author{Christos Dimitrakakis}
\email{christos.dimitrakakis@gmail.com}
\affiliation{%
  \institution{University of Neuchatel}
  \streetaddress{Av. du Premier-Mars 26}
  \city{Neuchatel}
  \country{Switzerland}
}

\renewcommand{\shortauthors}{T.\ K.\ Buening, M.\ Segal, D.\ Basu, A.\ M.\ George, C.\ Dimitrakakis}


\begin{abstract}
Typically, merit is defined with respect to some intrinsic measure of
worth. We instead consider a setting where an individual's worth is
\emph{relative}: when a Decision Maker (DM) selects a set of
individuals from a population to maximise expected utility, it is
natural to consider the \emph{Expected Marginal Contribution} (EMC) of
each person to the utility. We show that this notion satisfies an
axiomatic definition of fairness for this setting. We also show that
for certain policy structures, this notion of fairness is aligned with
maximising expected utility, while for linear utility functions it is
identical to the Shapley value. However, for certain natural policies,
such as those that select individuals with a specific set of
attributes (e.g. high enough test scores for college admissions),
there is a trade-off between meritocracy and utility maximisation.  We
analyse the effect of constraints on the policy on both utility and
fairness in extensive experiments based on college admissions and
outcomes in Norwegian universities.
\end{abstract}





\maketitle

\section{Introduction}
\label{sec:introduction}

Meritocracy~\citep{CED:meirtocracy} is the idea that individuals should be allocated
opportunities, resources and power in proportion to their talent,
abilities or achievements. However, it is hard to envisage an intrinsic
measure of individual worth: for example, merit for a particular
position depends on the required qualifications. When considering
filling multiple positions simultaneously, merit also depends on who
else has been selected. In this paper, we examine a definition of
meritocracy for such \emph{set selection} problems, where the decision
maker (DM) aims to maximise expected utility. Intuitively, meritocracy
can be defined in terms of how much each individual contributes to the
utility. As we discuss extensively in this paper, this intuition holds
true for particular utility functions, as long as the DM's selection
policy is not constrained by other fairness considerations.

Set selection problems appear in many settings where the DM must
select a subset from a candidate population, such as college
admissions or hiring decisions. The DM wishes to find a selection
policy that maximises utility in expectation. On the other hand,
meritocracy demands that individuals with higher merit have a higher
probability of being selected. Our first question is how to define
meritocracy in this setting.

In particular, as long as there is an inherent, static measure of
worth, what is or is not a merit-based (or meritocractic fair)
decision is well-defined.  This is typically the case when there is an
undisputed or dictated scoring system that induces a ranking over
individuals~\citep{young1994rise}.  In the absence of such a scoring
system, the DM may base the merit of an individual on the utility
function instead.  However, while a utility function over sets
supplies us with a definite best set (or collection thereof), it does
not generate individual-level judgements.

One possibility is to use the Shapley value~\cite{shapley}, a solution
concept from cooperative game
theory~\citep{chalkiadakis2011computational}, founded on the idea that
the potential contributions of an individual to the utility correspond
to their merit.
However, contributions to the utility can be relative and dependent on circumstances. More specifically, in this paper we focus on inherently non-linear utility functions defined over sets so that individual contributions to utility depend on who else is being selected. 
The following example illustrates this.
\begin{example}\label{example:example_2}
    In another scenario, Alice (A), Bob (B), Carlos (C) and David (D) may be candidates for an engineering team. Then, the overall utility we can extract from the team is not necessarily a simple linear function as the qualities of the individuals interact. For example, suppose that we wish to select a team of two and $D$ performs adequately only when paired with $C$, otherwise he is disruptive. We can model this by setting $\util(\{A, D\}) = \util(\{ B, D\}) = 0 $ and $\util(\{C, D\}) = 1$, where the function $\util$ denotes expected utility. Now, suppose that A works well together with B and adequately with C, i.e.\ $\util(\{A, B\}) = 2$ and $\util(\{A, C\}) = 1$. Let in all other cases the utility be zero. 
\end{example} 
The Shapley value for the utility function defined in Example~\ref{example:example_2} can be seen to be $\shapley(U)= ({}^1\!/\!{}_6, 0, 0, -{}^1\!/\!{}_6)$. This would suggest that a fair, merit-based selection should choose A with highest probability, B and C equally likely, and D with lowest probability. However, can we consider this fair?
Clearly, in Example~\ref{example:example_2}, a utility maximising DM will \emph{always} prefer to select A rather than D, as selecting A is guaranteed to yield at least as high utility as selecting D. 
This raises an important question:
Should the contribution to utility of an individual when paired with D be as strongly weighted as the contribution of said individual when paired with A?


We argue that since the actual contribution of an individual to utility depends on who else has been selected, \textit{the contribution of individuals should depend on the DM's selection policy.}
If the DM is maximising expected utility, then more useful sets have a higher probability of selection and individual contributions should be weighted more heavily in relation to those sets.  
Hence, unlike static measures of merit such as the Shapley value, we will view the utility-based merit of individuals, i.e. the potential contribution to utility, as a policy-dependent quantity.

\paragraph{Contributions.} This paper investigates the problem of selecting a set of individuals out of a candidate pool as an utility maximisation problem. In this setting, we define an individual's contribution to the utility as a dynamic measure of merit. We do this by  introducing the notion of {\em Expected Marginal Contribution} (EMC), modelling the potential contribution of an individual to the utility under a given policy. We then propose a  definition of {\em meritocracy} derived from and suited for arbitrary utility functions over sets and analyse its links to the EMC of individuals. We show that when the DM's policy is egalitarian, i.e.\ confers equal selection probability to every individual, the EMC is identical to the Shapley value. A natural way to move an egalitarian policy towards meritocracy is to increase the probability of selecting an individual according to their EMC. We show that this idea corresponds to a {\em policy gradient algorithm} for a specific class of policies, which are separably parameterised over the population. For this class of policies, we show that the policy gradient is a linear transformation of the EMC, and thus maximising utility also achieves meritocratic outcomes. 
While meritocracy per say does not guarantee other notions of fairness, the DMs utility function could include some fairness measure. This, however, decouples the EMC from the true merit of an individual. Another option is to add constraints to enforce fairness notions.
For \emph{constrained policies}, e.g.\ those that must select candidates through a parameterised function, or which have group fairness requirements, maximising utility does not yield meritocratic policies. Finally, we also perform experiments in a simulated college admission setting with a dataset derived from Norwegian university data. In this setup, we measure utility and meritocracy of the selection policies with and without group fairness and structural constraints. We show that while unconstrained policies maximise both, the constraints result in either reduced utility or meritocracy, which supports our theoretical results.



\section{Related Work}
\label{sec:related}



Most work on fairness in set selection has focused on policies that rank individuals according to some fixed criterion~\citep{kearns2017meritocratic, zehlike2017fa, celis2017ranking, biega2018equity, mathioudakis2019affirmative}. This approach satisfies meritocracy, since ``better''  and higher ranked individuals are preferred.
In particular, \citet{kearns2017meritocratic} consider a probabilistic ordering, generalising \citet{dwork2012fairness}'s notion of similar treatment to selection over multiple groups.
More precisely, a person $i$ in group $A$ is preferred to a person $j$ in group $B$ only if their relative percentile ranking is higher.
They extend this basic definition to different amounts of information available to the decision maker ranging from \emph{ex ante} to \emph{ex post} fairness.
\citet{singh2019policy} propose a fair ranking approach for Plackett-Luce models.
\citet{DBLP:journals/corr/abs-1801-03533} analyse a stylised parametric model of individual \emph{potential} and~\citet{celis2020interventions} consider interventions for ranking, where each individual has a latent utility they would generate if hired. \citet{emelianov2020fair} also
examine latent worth with variance depending on group (e.g.\ gender) membership.
Instead, we use a utility maximisation perspective, where meritocracy rewards individuals according to their contribution to utility, which depends on who else is selected.
This is in contrast to the above ranking methods, which implicitly assume a fundamental worth for individuals. In our setting, the contribution of each individual to the utility depends on who else is selected so that the worth of an individual depends directly on the DM's policy.

Fair set selection problems can also naturally be found in social choice, e.g.\ in participatory budgeting \citep{aziz2020participatory} and committee voting \citep{EFSS17,LaSk20}, as well as matching problems~\citep{hakimov2020experiments,manlove2013algorithmics}. However, in our setting the selected set is evaluated by a general utility function of the DM rather than by some specific aggregation of voters' preferences like in committee voting or participatory budgeting (for which adding another candidate is usually always beneficial).
Our work is more closely related to \citep{kusner2018causal}, which considered linear utilities, with the individual performance and group fairness depending on who else is selected. \citet{dwork2018group} and \citet{bairaktari2021fair} considered fairness-as-smoothness in \emph{cohort selection} for linear utilities. In contrast, we focus on a quite different question: how to define and ensure meritocratic fairness under non-linear utilities in set selection.

The EMC, to which our definition of meritocracy is strongly related, can be seen as a generalisation of the Shapley value~\citep{shapley}. The Shapley value measures the weighted average of individual (marginal) contributions over all possible sets. 
This measure, possibly most known from cooperative game theory, is often used in order to share compensations, costs or other utilities among agents according to their contribution to the entity in a game~\citep{Wint02}. 
An experimental study of reward allocations given by humans~\cite{EoLa20} found that the Shapley distribution is often not a natural choice as humans tend to weigh single-player coalitions more heavily and often violate some of the defining axioms of the Shapley value (while satisfying others). The EMC weights individuals' contributions to sets/coalitions based on a given policy's probability of selecting the set. For some policy spaces, iteratively adapting the policy in the direction of EMCs leads to a meritocratic policy. Whether a meritocratic distribution corresponds more closely to social norms than the Shapley value would be an interesting real-life study for future work.

\section{Setting and Notation} 
\label{sec:preliminaries}
\label{sec:setting}
We formulate the problem of selecting a set of individuals from a population from a general decision theoretic perspective, where the DM aims to select a subset of individuals maximising expected utility.

We consider a population of $\nPop$ candidates
$\pop = \{1, \dots, \nPop\}$.  The DM observes the features of the
population $\bx \in \mathcal{X}$ with
$\bx \triangleq (\bx_1, \ldots, \bx_N)$, makes a decision
$\act \in \Act$ about the population using a (stochastic) policy
$\pol(\act \mid \bx)$, observes an outcome $\by \in \CY$, and obtains
utility $\ut(\act, \by)$. Since the outcomes are uncertain, the
DM's goal is to maximise the \emph{expected} utility
$\util(\pol, \bx)$ given features $\bx$ and policy $\pol$. 

To make our results concrete, we focus on the case where $\Act = \{0,1\}^\nPop$ and interpret a decision $\act = (\act_1, \dots, \act_N)$ with $\act_i = 1$ and $\act_j = 0$ as selecting individual $i$ and rejecting individual $j$. We slightly overload notation and let $\pol(\act_i = 1 \mid \bx)$ denote the (marginal) probability of $i$ being selected under $\pol$ given $\bx$, i.e.\ $\pol(\act_i = 1 \mid \bx) = \sum_{\act \in \Act \, : \, \act_i = 1} \pol(\act \mid \bx)$. Similarly, $\pol(\act_i=0 \mid \bx)$ denotes the probability of $i$ being rejected under policy $\pol$. 

In our experiments (Section~\ref{sec:experiments}), $\CY$ is a product
space encoding outcomes for every individual in the population.  The
utility function is defined over sets and outcomes
$\ut : \Act \times \CY \to \Reals$, where $\ut(\act, \by)$ denotes the
utility of the selection $\act$ w.r.t.\ the outcomes $\by$.  In this
setting, the expected utility $\util(\act, \bx)$ of taking action
$\act$ given $\bx$ can be calculated by marginalising over outcomes:
\begin{align}
\util (\act , \bx) & \defn \E [\ut  \mid \act, \bx] =  \sum_{\by \in \CY} \Pr (\by \mid \act, \bx) \, \ut (\act, \by).
\label{eq:predicted_utility_set}
\end{align}
Here, $\Pr (\by \mid \act, \bx)$ is assumed to be a given predictive model used by the DM for outcome probabilities. We want to emphasise that the focus of this paper is not the fairness or bias of the predictive model $\Pr(\by \mid \act, \bx)$, but instead notions of meritocratic fairness based on the DM's utility function and selection policy.
Now, the expected utility $\util(\pol, \bx)$ of a policy $\pol$ given population $\bx$ takes the form:
\begin{align}
  \util ( \pol , \bx) & \defn \E_\pol [ \ut \mid \bx] = 
  \sum_{\act \in \Act} \pol (\act \mid \bx) \, \util (\act , \bx).
  \label{eq:policy-population-utility}
\end{align}
\textit{While $\util(\act, \bx)$ naturally induces a ranking over sets,
it does not necessarily provide a ranking over individuals} as each individual's contribution to utility may depend on the group selected alongside it.\footnote{For the majority of this paper, we state definitions and results with respect to $\util$, without loss of generality. To see this, consider deterministic outcomes.}

\paragraph{Problem formulation.}
The goal of the DM is to find a parameterised policy in the policy space
$\Pol = \cset{\pol_\vparam}{\vparam \in \Param}$ after observing a given population with features $\bx$
that maximises expected utility.
That is, we seek 
\begin{align}
  \label{eq:opt-population}
  \vparam^*(\bx) = \argmax_{\vparam \in \Param} \util (\pol_\vparam , \bx),
\end{align}
so that the chosen policy takes into account all the information $\bx$ we have about the current population.\footnote{The related problem of choosing a policy before seeing the current population is not treated in this paper.}




\section{Expected Marginal Contributions and Meritocracy}\label{sec:EMC and Meritocracy}
Here we consider the notion of expected marginal contribution of individuals for a specific policy and expected utility $U$. In the following, we will omit $\bx$ for brevity, since the expected utility is always conditional on $\bx$. In addition, the outcomes $\by$ and their distribution play no role in the following development.

Our definition of meritocracy is derived from the DM's utility and is
based on two stability axioms. We show that these stability notions  have meaningful links to
the expected marginal contribution of individuals. For formal proofs of the statements in this section, we refer to \cref{app:proofs from sec 4}.

\subsection{The Expected Marginal Contribution}
\label{sec:exp_marg_contr}

In the following, we will use $\util (\act) \equiv \util(\act, \bx)$ to denote the expected utility of selecting $\act$, and $\util (\act + i)$ of adding the individual $i$ to the selection, i.e. $\act + i = (\act_1, \dots, \act_{i-1}, 1, \act_{i+1}, \dots , \act_N)$. Their difference, $\util (\act + i) - \util (\act)$ can be seen as the marginal contribution of the individual to the set $\act$.

We generalise this to the {\em Expected Marginal Contribution} (EMC) of an individual $i$ under a policy $\pol$:
\begin{align}\label{eq:expected marginal contribution}
    \emc_i (\util, \pol) \defn \sum_{\act \in \Act} \pi(\act) \big[ \util(\act + i) - \util(\act) \big],
\end{align}
where the implicit dependence on $\bx$ has been dropped for brevity.

Informally, the EMC of individual $i$ corresponds to the gain (or loss) in expected utility when a policy is modified so as to always pick individual $i$.
\ifdefined \LongVersion
If we are given a distribution $\beta$ over populations, then we define the expected marginal contribution with respect to that distribution as
\begin{align*}
    \E_\beta [\EMC_j \util (\pol, \bx)] = \int_\bx \Delta_j \util (\pol, \bx) \beta(\bx).
\end{align*}
\fi
The concept of individual contributions to utility has been studied in cooperative game theory, where the celebrated Shapley value~\citep{shapley} constitutes a fair resource allocation based on marginal contributions. The Shapley value of a utility function $\util$ (also called a characteristic function) is defined as
\begin{align*}
    \shapley_i (\util) =  \frac{1}{N} \sum_{\act \in \Act \colon  \act_i = 0} \binom{N-1}{\lVert \act \rVert_1}^{-1} \big[\util(\act + i) - \util(\act)\big]. 
\end{align*}
In particular, we see that the Shapley value corresponds to the EMC for an egalitarian policy under which every individual has the same probability of being selected. 
\begin{remark}\label{lemma:shapley egalitarian}
Under the egalitarian selection policy 
$\egpol(\act) = \frac{1}{N \binom{N-1}{\lVert \act \rVert_1}}$, we have $\emc (\util, \egpol) = \shapley(U)$.
\end{remark}
\noindent
To emphasise the difference between the Shapley value and the EMC, let us revisit the set selection problem of Example~\ref{example:example_2}.

\begin{example}[Example \ref{example:example_2} continued]
Recall the scenario from Example \ref{example:example_2} with applicants Alice (A), Bob (B), Carlos (C) and David (D) as well as expected utilities $\util(\{A, B\}) = 2$, $\util(\{A, C\}) = 1$, $\util(\{C, D\}) = 1$, and $\util(S) = 0$ for all other sets $S$. 
In this situation, A has (arguably) the highest merit and D the lowest merit.  
In particular, a utility maximising DM would {always} prefer to select A over D. 
This is also expressed in the Shapley value and the EMC under the egalitarian policy $\egpol$: \  
$\shapley(\util) = \emc(\util, \egpol) = ({}^1\!/\!{}_6, 0, 0, -{}^1\!/\!{}_6)$. Now, let us account for this and consider the uniform policy that would never select D without also selecting A: 
\begin{align*}
    \pol(\act) = 
    \begin{cases} 0, \text{ if } \act_A = 0 \wedge \act_D = 1 \\ \frac{1}{12}, \text{ otherwise.}
    \end{cases}
\end{align*}
Under policy $\pol$, we see that the EMCs of $(A,B,C,D)$ 
are given by $\emc (\util, \pol) = ({}^3\!/\!{}_{12}, {}^1\!/\!{}_{12}, -{}^1\!/\!{}_{12}, -{}^2\!/\!{}_{12})$. Thus, in contrast to the Shapley value (which is oblivious to $\pol$), the EMC endows B with higher merit than C; accounting for the fact that A should always be preferred over D. We see that as the DM changes their policy towards sets with high utility, the contributions of individuals shift as well and we may have to rethink an individual's utility-based merit. 

\end{example}
The Shapley value~\citep{shapley} is characterised by four desirable axioms of fair division: symmetry, linearity, the treatment of null players, and efficiency. In fact, we see that the EMC satisfies analogous axioms, except for efficiency, which does not apply to the EMC as it lacks a corresponding normalisation. 

\begin{lemma}[Axioms of Fair Division]\label{lemma:axioms_of_fair_division}
    The EMC satisfies 
    \begin{enumerate}
    \item[1)] \textbf{Symmetry:} 
    If utility function $\util$ and individuals $i,j \in \pop$ are such that  $\util(\act + i) = \util(\act + j)$ for all $\act \in \Act$, then $\emc_i (\util, \pol) = \emc_j (\util ,\pol)$ for all  policies $\pol \in \Pol$.
    
    \item[2)] \textbf{Linearity:} 
    If $\util_1$ and $\util_2$ are two utility functions, then $\emc (\alpha \util_1+ \beta \util_2 ,\pol) = \alpha \emc (\util_1 ,\pol) + \beta \emc (\util_2, \pol) $ for all policies $\pol \in \Pol$ and $\alpha, \beta \in \Reals$.
    
    \item[3)] \textbf{Null Players:} 
    If $i \in \pop$ has zero contribution to every set, i.e.\ $\util(\act+i)=\util(\act)$ for all $\act \in \Act$, then $\emc_i (\util, \pol) = 0$ for all policies $\pol \in \Pol$.
\end{enumerate}
\end{lemma}






\subsection{Meritocracy and Stability Criteria} 
\label{sec: meritocracy}

We now define a notion of meritocracy for general utilities over candidate sets. This definition of meritocratic fairness is based on two properties of the selection policy, namely, swap stability and local stability. We will show that these properties have natural links to the EMC, which suggests that the EMC can be understood as  the additional reward individuals should obtain under a policy. 
\begin{definition}[Swap Stability]\label{definition:swap-stability}
A policy $\pol$ is called {\em swap stable} if for any two individuals $i,j \in \pop$ with $\pol(\act_i = 1) > \pol(\act_j = 1)$,  we have $\util(\pol+i-j) \ge \util(\pol-i+j)$, where $\util(\pol + i-j) \defn \sum_{\act \in \Act} \pol(\act)\util(\act+i-j)$.
\end{definition}
In other words, a policy $\pol$ is swap stable if for any two individuals $i, j \in \pop$, from which $i$ is more likely to be selected, the utility of selecting $i$ but not $j$ is higher than the utility of selecting $j$ but not $i$.
Note that instead of writing $\util(\pol+i-j)$, we can also write $\util(\pol')$, where policy $\pol'$ shifts all the mass to sets that include $i$ but not $j$, i.e.\ $\pol'(\act)= \pol(\act)+\pol(\act-i)+\pol(\act+j)+\pol(\act-i+j)$ if $a_i=1$ and $a_j=0$, and $\pol'(\act)=0$ otherwise.
We establish the following link to the EMC, where we use that $\emc_i (\util, \pol - i -j) = \sum_{\act \in \Act} \pol (\act) [\util(\act +i -j) - \util(\act - i - j)]$.


\begin{lemma}\label{lemma:swap-stability equivalence}
For $i,j \in \pop$ and policy $\pi$, $\util(\pol+i-j) \ge \util(\pol-i+j)$ if and only if $\emc_i (\util, \pol-i-j) \ge \emc_j (\util, \pol-i-j)$.
\end{lemma}

Hence, for a swap stable policy, the merit of  $i$, measured by the EMC of $i$ when selecting neither $i$ nor $j$, is higher than that of $j$. 
Next, we define locally stable policies. Intuitively, these have the property that the DM rewards individuals as much as their self-interest allows them to. 
In other words, if an individual could contribute positively to the utility, the individual ought to be selected with higher probability.
\begin{definition}[Local Stability]\label{definition:local-change-stability}
A policy $\pol$ is \emph{locally stable} if for any $i \in \pop$, we have $\util(\pol) \ge \util(\pol+i)$, where $\util(\pol+i) \defn \sum_{\act \in \Act} \pol(\act)\util(\act+i)$.
\end{definition}

Instead of $\util(\pol+i)$, we can once again write $\util(\pol'')$, where policy $\pol''$ shifts all the mass to sets that include $i$, i.e. $\pol'(\act)= \pol(\act)+\pol(\act-i)$ if $a_i=1$, and $\pol''(\act)=0$ otherwise.
Thus, stability under local changes guarantees that the expected utility of selecting $i$ is lower than the expected utility of $\pol$.
In fact, we see that this is equivalent to $\emc_i (\util,\pol) \leq 0$.
\begin{lemma}\label{lemma:local-change-stability equivalence}
For $i \in \pop$ and policy $\pi \in \Pol$ we have $\util(\pol) \ge \util(\pol+i)$ if and only if \, $\emc_i (\util,\pol) \leq 0$. 
\end{lemma}
\noindent
We now define a meritocratic policy as the one satisfying both stability properties.
\begin{definition}[Utility-Based Meritocracy]
A policy $\pol$ is called \emph{meritocratic} if it is swap stable and locally stable.
\end{definition}
This notion of meritocracy ensures that our policy is fair to individuals in the sense that individuals with lower chances of being selected have no justification of increasing their chances or swapping their chances of being selected with another individual under a meritocratic policy.

\begin{example}[Example \ref{example:example_2} continued]
Once again, let us recall Example \ref{example:example_2} with applicants Alice (A), Bob (B), Carlos (C) and David (D) as well as expected utilities $\util(\{A, B\}) = 2$, $\util(\{A, C\}) = 1$, $\util(\{C, D\}) = 1$, and $\util(S) = 0$ for all other sets $S$. \\
The uniform policy $\pol(\act) = {}^1\!/\!{}_{16}$ for all $\act \in \Act$ is vacuously swap stable as all individuals have equal probability of selection. However, $\pol$ is not locally stable as $\util(\pol) = {}^4\!/\!{}_{16}$, whereas $\util(\pol + A) = {}^6\!/\!{}_{16}$. In contrast, the policy that selects the set $\{C, D\}$ with probability one is meritocratic. While $\{C, D\}$ is not utility maximising, the selection is locally optimal in the sense that changing only $C$ or $D$ does not yield strictly larger utility. 
In particular, we see that selecting the "globally" optimal set (in this case $\{A, B\}$) with probability one is always meritocratic fair according to our definition. 
\end{example}

\begin{lemma}\label{lemma:determinstic_optimal_meritocractic}
Any deterministic policy $\pi$ that maximises expected utility
is meritocratic. 
\end{lemma} 


Whereas local stability is satisfied by any optimal policy (cf.\ Proof of Lemma \ref{lemma:determinstic_optimal_meritocractic}), we find that swap stability generally does not hold for {\em stochastic} optimal policies. 

\begin{lemma}\label{lemma:stochastic_optimal_not_meritocractic}
A stochastic policy $\pol$ that maximises expected utility may not satisfy swap stability.
\end{lemma}
 
The proof of Lemma \ref{lemma:stochastic_optimal_not_meritocractic}, which can be found in Appendix \ref{app:proofs from sec 4}, relies on a counterexample. The counterexample suggests that members of utility maximising sets (or optimal sets) can have different merit by being more (or less) compatible across these optimal sets. This suggests that for the swap stability to hold the selection policy has to satisfy one of the following conditions. First, the DM needs to select optimal sets in a manner that gives each optimal set member the same selection probability, thereby avoiding comparison all together. Alternatively, the DM needs to select individuals with similar contributions across {\em all} the optimal sets similarly often. 

The latter condition leads to the concept of individually {\em smooth} policies in the sense that similar individuals (in terms of compatibility across optimal sets) are being selected similarly often. This might be another desirable property of meritocratic fair decisions as it would favour individuals that do well in a variety of optimal groups.
Note that in general stochastic optimal policies could be preferred over deterministic ones, since from a group-level perspective, it appears unfair to select any single set with probability one in the presence of other sets with equal utility. More generally, the concept of smoothness could also be applied to the selection of suboptimal individuals, which would result in smooth decisions over the whole population (and not only optimal individuals). 

\section{Policy structures and optimisation}
\label{sec:util-maxim-marg}
Here we consider how different constraints on the policy structure affect the utility and meritocracy of optimal policies. Section~\ref{sec:separable} shows that when the DM is can decide to accept or reject each individual separately, maximising utility is meritocratic. However, frequently the DM is constrained to apply a uniform criterion for acceptance or rejection, such as a threshold, as we discuss in Section~\ref{sec:threshold}. However, such policies 

In particular, Section~\ref{sec:separable} shows that the policy gradient $\nabla_\vparam \util(\pol_\vparam, \bx)$ is a linear transformation of the EMC for specific separably parameterised policies. This demonstrates a connection between meritocracy and utility maximisation via the EMC of individuals by the results of the previous section.
Threshold policies, in contrast, may not always enjoy this property. In either case, 
optimal policies can be found through a policy gradient algorithm (\cref{sec:policy gradient}).
For brevity, we again omit the feature vector $\bx$ here and remind ourselves that all policies $\pol(\act \mid \bx)$ as well as expected utilities $\util(\pol, \bx)$ and $\util(\act, \bx)$ are conditional on $\bx$.


\subsection{Separable Policies}\label{sec:separable}
We say that a parameterised policy $\pol_\vparam$ is {\em separable over the population} $\mathcal{N}$ if $\vparam = (\param_1, \dots, \param_N)$ and $\pol_\vparam(\act) = \prod_{i=1}^N \pol_{\vparam}(\act_i)$, where the probability of selecting individual $i$ takes form
\begin{align}\label{eq:definition separable}
    \pol_{\vparam} (\act_i) = \frac{g(\act_i, \param_i)}{Z(\vparam)}
\end{align}
for some function $g$ and normalisation $Z(\vparam)$. 
This means that the probability of a selection $\act$ can be factorised in terms of individual parameters $\param_i$ and a common denominator depending on $\vparam$. 

\paragraph{Softmax policies.}
A natural choice for policies that take form as in \eqref{eq:definition separable} are softmax policies of the following kind. For $\beta \geq 0$ and $\vparam = (\param_1, \dots, \param_N)$, we define
\begin{align}\label{equation:softmax_policy}
    \pol_\vparam (\act) = \frac{e^{\beta  \vparam^\top \act} }{\sum_{\act^\prime \in \Act} e^{\beta \vparam^\top \act^\prime}}.
\end{align}
Here, $\beta \geq 0$ is called the inverse temperature of the distribution. 
While such policies have the advantage that the probability of selecting a set is naturally constrained, they are clearly impractical for large scale experiments as calculating the denominator is computationally heavy.
However, we get a first glance at the intimate relationship between the policy gradient and the EMC of individuals. 
\begin{lemma}\label{lemma:softmax no. 2}
The gradient of the softmax policy $\pol_\vparam$ as defined in \eqref{equation:softmax_policy} is a linear transformation of the EMC. More precisely, for every $i \in \pop$ we have
$
    \nabla_{\param_i} \util (\pol_\vparam) = \beta \pol_\vparam (\act_i = 1) \, \emc_i (\util, \pol_\vparam)
$. 
\end{lemma}
Note that the EMC of individual $i$ under policy $\pol$ constitutes the expected gain (or loss) in expected utility of always adding individual $i$ under $\pol$. Consequently, if individual $i$ is being selected with probability $1$, then $\emc_i (\util, \pol_\vparam)$ is zero. More generally, the EMC of an individual $i$ is decreasing as the probability of selecting individual $i$ is increasing. For this reason, we can understand the factor $\pol_\vparam (\act_i = 1)$ in Lemma \ref{lemma:softmax no. 2} as a normalisation of the EMC opposing this effect, as we would expect the policy gradient to be agnostic about the current share of individual $i$.

\begin{figure*}[t!]
\centering
\begin{subfigure}{.33\textwidth}
  \centering
  \includegraphics[width=\textwidth,trim={3cm 8cm 4cm 9cm},clip]{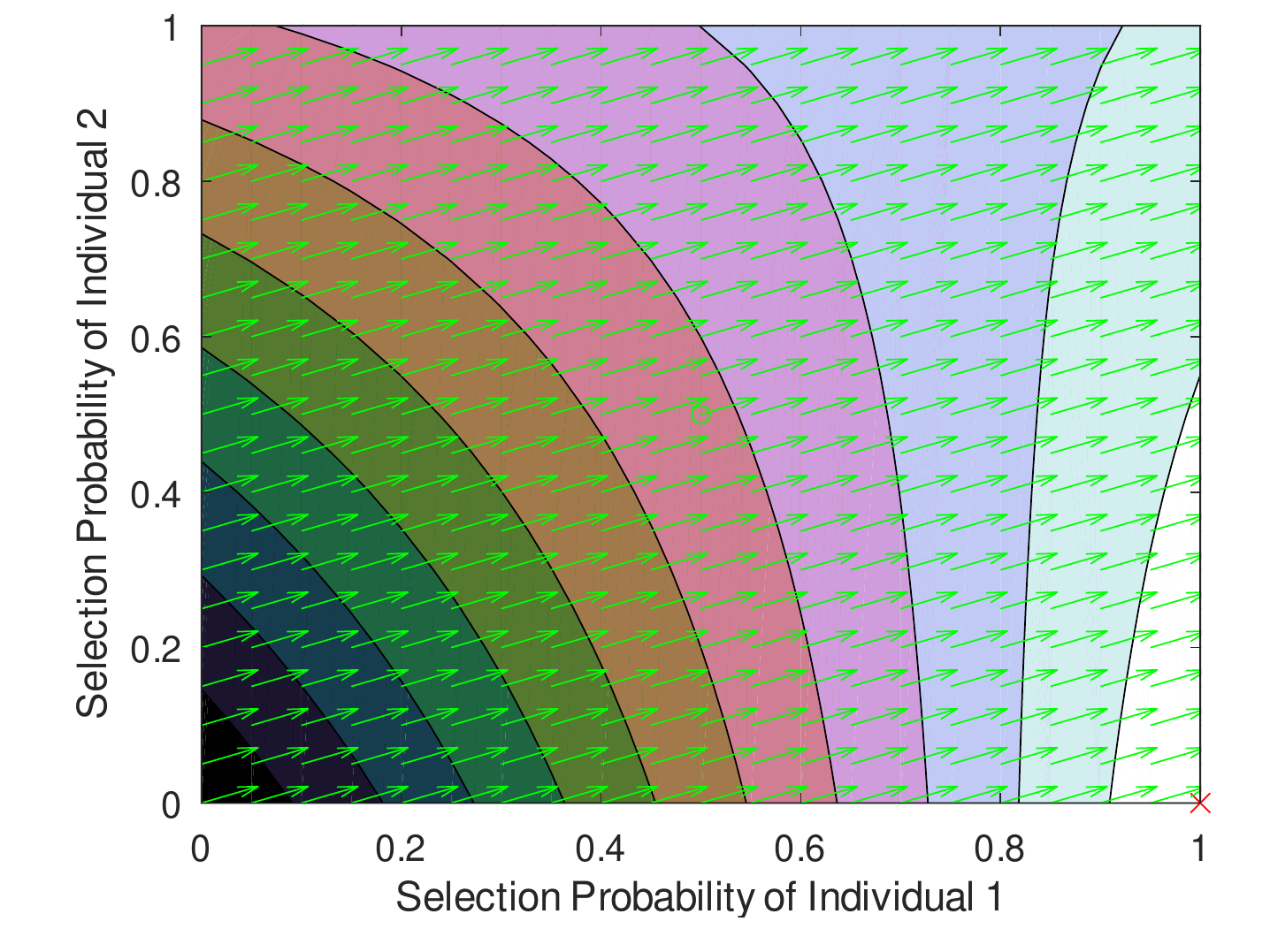}
  \caption{Shapley value}
  \label{fig:utility_landscape_Shapley}
\end{subfigure}
\begin{subfigure}{.33\textwidth}
  \centering
  \includegraphics[width=\textwidth,trim={3cm 8cm 4cm 9cm},clip]{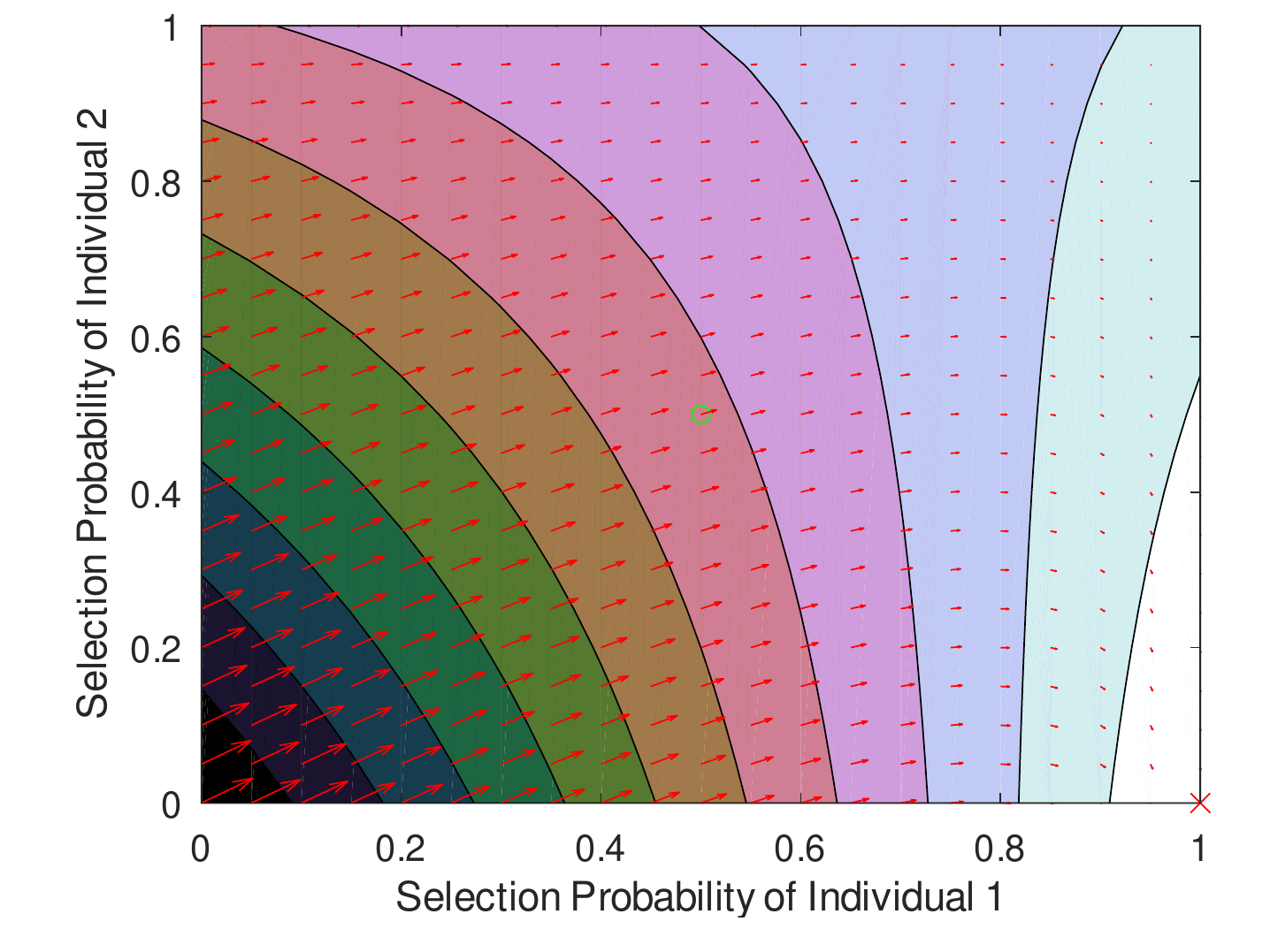}
  \caption{EMC}
  \label{fig:utility_landscape_EMC}
\end{subfigure}
\begin{subfigure}{.33\textwidth}
  \centering
  \includegraphics[width=\textwidth,trim={3cm 8cm 4cm 9cm},clip]{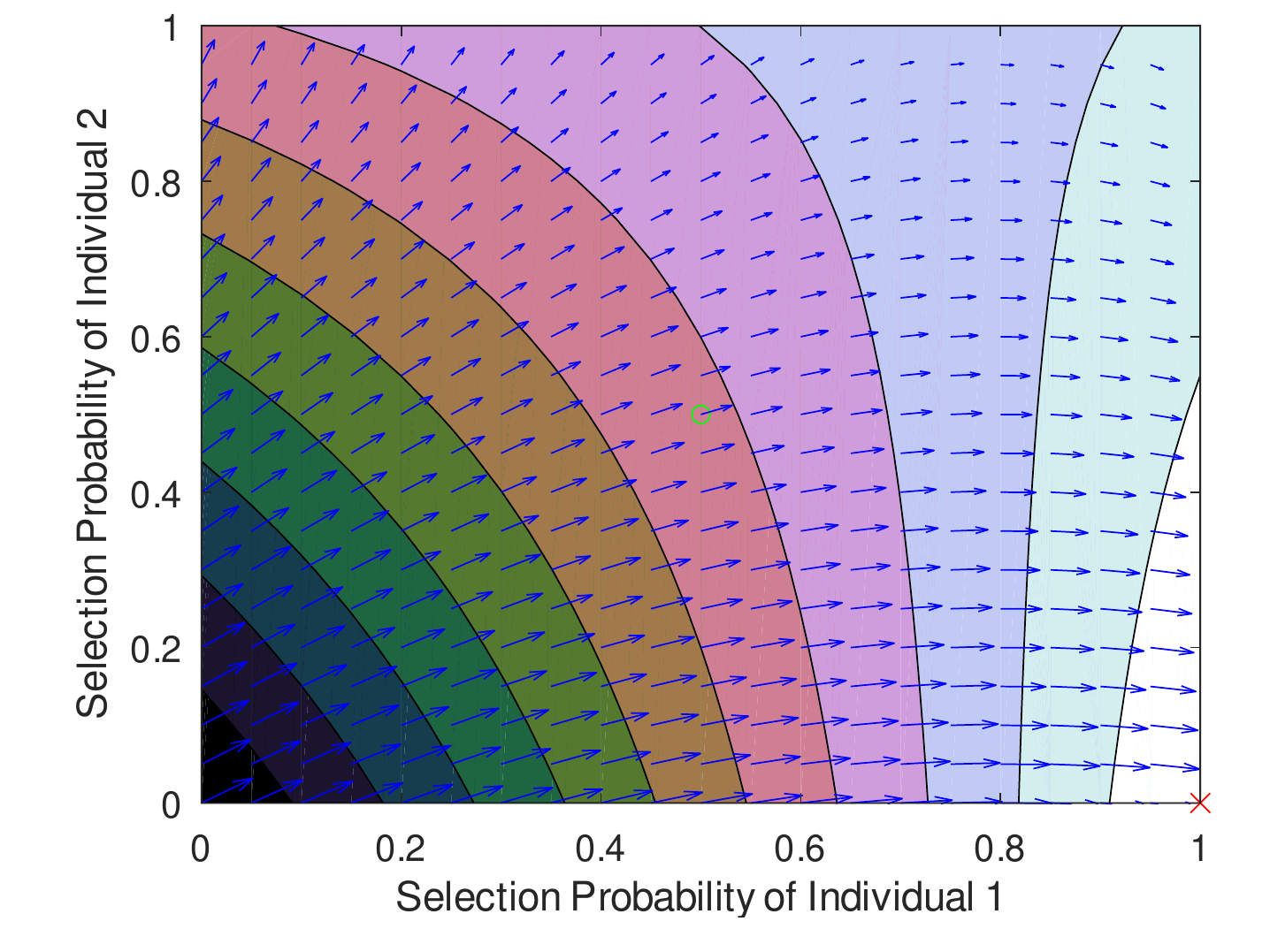}
  \caption{Policy Gradient}
  \label{fig:utility_landscape_gradient}
\end{subfigure}%
\caption{Shapley value, EMC and policy gradient for two individuals, a separable linear policy, a log-linear utility, and a selection cost of $c=0.3$ (cf.\ \cref{sec:experiments}). Light colours denote higher utility, dark colours lower utility. The maximum is indicated by a red cross and the egalitarian policy $(0.5, 0.5)$ by a green circle. We see that at $(0.5, 0.5)$, the Shapley value, EMC and policy gradient point in the same direction. While the Shapley value points to the same direction everywhere, the EMC changes directions depending on the current policy and so does the policy gradient. In particular, the EMC tends to zero as the probability of selecting both individuals tends to one. As we will see in Lemma \ref{lemma:softmax no. 2} and Lemma \ref{lemma:multinomial gradient} the policy gradient is closely related to the EMC. In particular, across the diagonal $y=x$ both point in the same direction.}
\label{fig:utility_landscape}
\end{figure*}




\paragraph{Linear policies.} 
From a computational point of view, it is appealing to consider separable policies so that the probability of selecting individual $i$ truly depends on the parameter $\param_i$ only, i.e.\ $Z(\vparam) = const$ in equation \eqref{eq:definition separable}. Then, the probability of selecting or rejecting individual $i$ is given by $\pol_\vparam(\act_i) = \pol_{\param_i} (\act_i)$. We want to emphasise that such policy structures do not render decisions about individual $i$ independent from those about individual $j$. For instance, the expected gain (or loss) $\util(\pol + i)$ from adding individual $i$ might be small or large depending on whether individual $j$ is likely to be included under $\pol$. 
We now introduce separable linear policies as these have a particularly intuitive structure. Separable linear policies select individual $i$ with probability $\param_i$, i.e.\ 
\begin{align*}
    \pol_{\param_i} (\act_i) = \param_i \ind{\act_i = 1} + (1- \param_i ) \ind{\act_i = 0}.
\end{align*}
A decision $\act \in \Act$ is then assigned the probability $\pol_\vparam (\act) = \prod_{i = 1}^N \pol_{\param_i}(\act_i)$.
We again find that there is a natural link between the policy gradient and the EMC of individuals.

\begin{lemma}\label{lemma:multinomial gradient}
For separable linear policies, if $\pol_{\param_i} (\act_i = 0 ) > 0$, then 
$\nabla_{\param_i} \util (\pol_\vparam) = \frac{\emc_i (\util ,\pol_\vparam)}{\pol_{\param_i} (\act_i = 0)}$ for every $i \in \pop$. 
\end{lemma}
We observe that the gradient has a similar form as in Lemma \ref{lemma:softmax no. 2}, where in this case the ``normalising'' factor is the reciprocal of $\pol_{\param_i} (\act_i = 0)$. 
In particular, we see that for separable linear policies as well as separable softmax policies the policy gradient is equal to the (uniformly scaled) EMC when evaluated at any egalitarian policy that gives the same selection probability to all individuals. 
The relation of the Shapley value, EMC and policy gradient is also exemplarily illustrated in \cref{fig:utility_landscape}. 

\subsection{Threshold Policies}
\label{sec:threshold}
While separable policies can make arbitrary decisions, which may be undesirable from the point of view of fairness, we here consider a class of policies that apply the same decision rule to every individual. This uniformity of treatment should lead to similar outcomes for similar individuals. While this may not make sense when hiring a team of experts, it is eminently suitable for college admission settings where transparent and easily interpretable decision rules are preferable.
We define these policies by parameterising over the feature space $\CX$. One natural example is a policy of logistic type, where $\vparam = (\param_1, \dots , \param_{|\mathcal{X}|})$ and the probability of selecting individual $i$ is given by
\begin{align*}
    \pol_\vparam (\act_i = 1 \mid \bx_i) = 
    \frac{e^{\vparam^\top \bx_i}}{1+e^{\vparam^\top \bx_i}}.
\end{align*}
Such policies can be understood as threshold policies, since an individual $i$ with feature vector $\bx_i$ such that $\vparam^\top \bx_i > 0$ is being selected at least $50$\% of the time. In practice, these policies can be projected to the next closest vertex in the simplex so as to attain a deterministic threshold, namely, individual $i$ is being selected if only if $\vparam^\top \bx_i > 0$.  


\subsection{A Policy Gradient Algorithm}
\label{sec:policy gradient}
A natural way to find an optimal policy parameter $\vparam^*$ is to use policy gradient (\cref{alg:pol_opt}).
For policies that are separably parameterised over the population, this naturally shifts the DM's policy towards meritocratic fair decisions as the policy gradient is a linear transformation of the EMC.
\begin{algorithm}[H]
	\caption{Policy gradient algorithm}
	\label{alg:pol_opt}
	\begin{algorithmic}[1]
		\State \textbf{Input}: a model $\Pr(\by | \act, \bx)$, a population $\pop$ with features $\bx$ and a utility function $\ut$.
		\State \textbf{Initialise}: $\vparam^0$, $\delta > 0$, learning rate $\eta > 0$
		\While{$\| \vparam^{i+1} - \vparam^i \| > \delta$}
		\State Evaluate $\nabla_\vparam \util (\pol_\vparam , \bx) $ from $\ut$, $\bx$ and $\Pr$.
		\State $\vparam^{i+1} \leftarrow \vparam^i + \eta [\nabla_\vparam \util (\pol_\vparam , \bx) ]_{\vparam=\vparam^{i}}$
        \State $i+\!+$
		\EndWhile
		\State \textbf{return} $\pi_{\vparam^{i+1}}$
	\end{algorithmic}
\end{algorithm}


We visualise the policy gradients and associated EMCs in \cref{fig:utility_landscape}. They have similar directions except for the edges where the probability of selecting an individual is almost $1$. This illustrates the effect of the ``normalising'' factors emerging in Lemma \ref{lemma:softmax no. 2} and Lemma \ref{lemma:multinomial gradient}. We also observe that the Shapley value, the EMC, and the policy gradient are similar near the point $(0.5, 0.5)$, i.e.\ when the policy is almost egalitarian. This validates our arguments of the policy gradient reducing to the EMC and in turn the EMC reducing to the Shapley value under an egalitarian policy.

\section{Case Study: College Admissions in Norway} \label{sec:experiments}
In this section, we perform an empirical case-study on the selection of applicants from a candidate pool in a simulated college admission system based on real-world student data from Norway. Recall that deterministic selection of a utility maximising set always satisfies our notion of meritocracy (Lemma \ref{lemma:determinstic_optimal_meritocractic}). For this reason, we also consider the setting where we impose an $\varepsilon$-statistical parity constraint on the optimisation and analyse its effect on the meritocracy and utility of the decisions. 

\subsection{Data and Experimental Setup}
\paragraph{Data.} We use two datasets from the Norwegian Database for Statistics on Higher Education:
\begin{enumerate}
    \item Application data\footnote{\url{https://dbh.nsd.uib.no/dokumentasjon/tabell.action?tabellId=379}} of all applicants to all Norwegian university programs including the following details:
    birth date, semester of application, gender,  citizenship, country of educational background,
    high school grades in form of GPA and summarised language/science points, other points, admission decision, and each applicant's preference for the program. 
    \item Exam data\footnote{\url{https://dbh.nsd.uib.no/dokumentasjon/tabell.action?tabellId=472}} of all students at Norwegian universities for all their taken exams including: courses, study program, and achieved grades.
\end{enumerate}
We use data from the time period of $2017-2020$ and limit ourselves to one 5-year master's program with a large number of admitted students at one Norwegian university. In terms of the outcomes, we consider exam data for three mandatory courses in the same program.

\paragraph{Simulator}
As our setting is interactive\footnote{Specifically, the exam outcomes depend on our actions, i.e. which students have been admitted, which are fixed in the historical data.}, we cannot use a static dataset for evaluation. Instead, we use the data to construct a simulator, which generates populations and outcomes for our experiments (for more details see Appendix~\ref{app:details_of_experiments}).
We simulate applications for the considered 5-year master's program using ctgan \citep{xu2019modeling} 
with the following features: age (at application), gender, citizenship, country of educational background, high school GPA,
science points and language points (from the high school degree), other points and priority (applicant's preference for the specific program).
The part of the simulator producing exam outcomes uses linear regression trained on the exam outcomes of admitted students in the original data.

\paragraph{DM's motivation}
We assume that the DM (the faculty) is interested in having successful students across all three courses/disciplines. This could, for instance, be motivated by the wish to produce capable graduates for all necessary fields of practice, so as to uphold the department's reputation.

\paragraph{Data generation.}
We first generate records for 5,000 students, which represent the admitted students from previous years that have graduated, using the simulator described above and an admission policy described in the supplementary material.
Next, we sample again $200$ students
from the same simulator. This time, the DM does not observe the course results of the students. These students represent this year's applicants from which the DM has to choose. 

\paragraph{The DM's regression modelling.}
Since the DM's utility is tied to individual outcomes, they must estimate a model $\Pr(\by \mid \bx, \act)$ for the course results $\by$ of this year's applicants $\bx$ in order to make their selection. 
The DM's model (which is of course different from the one the simulator uses) is estimated using regression on the data from the 5,000 simulated students admitted in prior years.
 
\paragraph{Utility function and constraint.}
We assume that the DM is interested in good course results of admitted students across all three considered disciplines/courses. 
In addition, the DM has a cost $c$ associated with admitting a student. 
Then, the DM would only select a student if the gain in utility from adding the student to the selection exceeds the cost $c$ of selecting them.
For any given set of {\em outcomes} of the admitted applicants, we design a log-linear utility that values balanced course results across all areas\footnote{The potential outcomes of non-admitted applicants do not contribute to the utility}:
\begin{align}\label{eq:definition_utility_function}
    \ut(\act, \by) =  \sum_{j=1}^3 \log \left( \sum_{i \in \pop} \act_i \cdot \by_{i,j} \right) - c \cdot \| \act \|_1.
\end{align}
Since future  results $\by$ are unknown to the DM, they instead maximise expected utility by marginalising the above utility function $\ut$ over the outcomes as described in equation~\eqref{eq:predicted_utility_set}.  

In addition to the unconstrained optimisation problem, we are interested in the effect that additional constraints have on the meritocracy of the selection policy. 
Constraints could be needed in order to enable educational opportunities for a minority group. Consider the case where a minority and majority group perform well on their own but have very limited utility when members are mixed. Here we would only select members of the majority group.
We apply a typical notion of group fairness, namely, $\varepsilon$-statistical parity with respect to the feature gender. Then, a selection policy $\pol$ must satisfy
\begin{align}
  | \pol(\act_i = 1 \mid i \text{ is male}) - \pol(\act_j = 1 \mid j \text{ is female}) | \leq \varepsilon .
 \label{eq:constraint}
\end{align}

\paragraph{Algorithmic comparisons.}
We test the policy gradient approach for separable linear and threshold policies as introduced in \cref{sec:util-maxim-marg}. For the threshold policies, we use all high-school grades available (grade points, language points, science points and other points) as a numerical four dimensional feature. This policy simulates the case in which the DM provides an interpretable policy with predetermined admission criteria. 
For both policy structures, we allow the policy gradient algorithm 250 updates to converge. 
In addition, we use the uniformly random selection of sets as a trivial lower baseline. We also run a stochastic greedy algorithm (see e.g.\ \cite{mirzasoleiman2015lazier}), which is a robust algorithm for finding the utility maximising set in \emph{unconstrained} set selection problems. For the constrained optimisation problem, we use adaptive penalty terms to ensure that the policy gradient algorithm satisfies the constraint (see \cref{app:details_of_experiments}). To obtain a selection satisfying the constraints from the stochastic greedy algorithm, we set the utility of all sets violating the constraint to zero. In particular, note that the stochastic greedy algorithm is usually deployed for unconstrained utility maximisation problems only so that one can expect it to achieve comparably low utility in the constrained case.

\subsection{Measuring Deviation from Meritocracy}
We are interested in the tensions that arise between utility maximisation and meritocratic fair decisions. To this end, we aim to quantify the {\em deviation from meritocracy} in our experiments. To measure violations of swap stability as introduced in  Definition \ref{definition:swap-stability}, we suggest to use 
\begin{eqnarray*}
  &&\hspace*{-18pt}  \devswap(\pi, \bx) = \sum_{i, j \in \pop} (\pol(\act_i = 1 \mid \bx ) - \pol(\act_j = 1 \mid \bx))^+ \big( \util( \pol - i + j, \bx) \\
&&\hspace*{-18pt}\qquad - \util( \pol +i -j, \bx) \big)^+,
\end{eqnarray*}
where $(X)^+ \defn \max \{0, X\}$. Here, large values of $\devswap(\pol, \bx)$ indicate large deviations from swap stable decisions.
Note that our choice of $\devswap$ not only accounts for the number of infringements, but also the magnitude of the deviation from swap stability. For instance, if $\util (\pol +i - j) \ll \util(\pol-i +j)$ while $\pol(\act_i = 1) \gg \pol(\act_j = 1)$, the measured deviation from swap stability is accordingly large. In particular, if $\devswap (\pol, \bx ) = 0$, the policy $\pol$ is swap stable. 
To measure the deviation from local stability (Definition \ref{definition:local-change-stability}), we use the cumulative positive EMCs under policy $\pol$:
\begin{align*}
    \devchange(\pi, \bx) = \sum_{i \in \pop} \big(\emc_i (\util, \pol, \bx) \big)^+.
\end{align*}
Again, larger values of $\devchange(\pol, \bx)$ indicate more severe deviations from locally stable decisions. Recall that   Lemma \ref{lemma:local-change-stability equivalence} states that a policy $\pol$ is stable under local changes if and only if $\emc_i (\util ,\pol, \bx) \leq 0$ for all $i \in \pop$. Thus, if $\devchange(\pi, \bx) = 0$, the policy $\pol$ is stable under local changes.\footnote{Clearly, one could define deviation from meritocracy differently, e.g.\ by counting the number of violations. However, we choose to use $\devswap$ and $\devchange$ as they also give insights into the severity of the deviation.}

\subsection{Experimental Analysis}

\begin{figure*}[t!]
\centering
\begin{subfigure}{.49\textwidth}
  \centering
  \includegraphics[width=\textwidth]{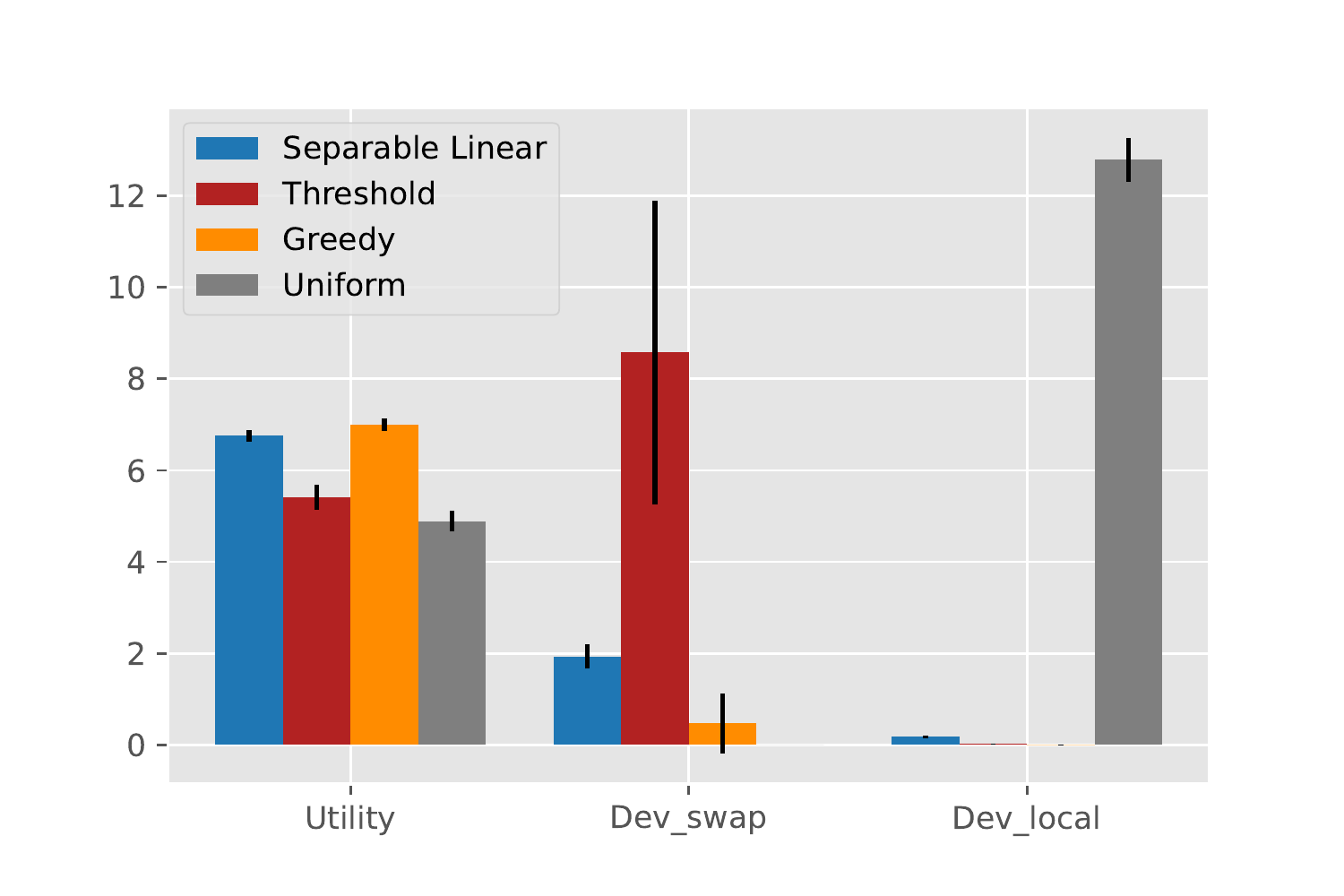}
  \caption{Unconstrained Optimisation}
  \label{fig:unconstrained_results}
\end{subfigure}
\begin{subfigure}{.49\textwidth}
  \centering
  \includegraphics[width=\textwidth]{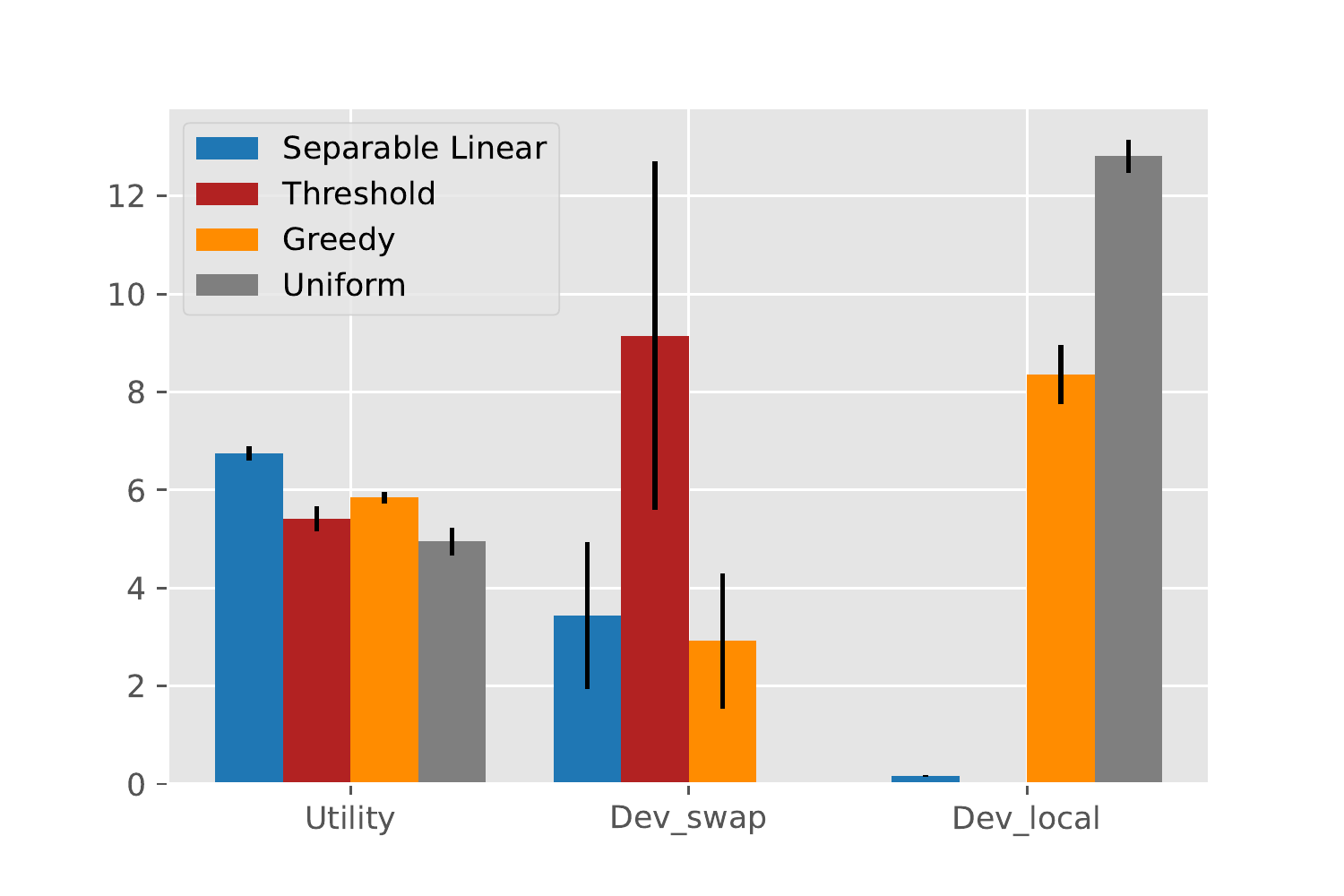}
  \caption{Optimisation under $\varepsilon$-statistical parity}
  \label{fig:constrained_results}
\end{subfigure}
\caption{Expected utility, $\devswap$ and $\devchange$ with respect to the {\em true outcomes} for log-linear utility with cost $c=0.05$ and $200$ applicants. (a): unconstrained utility maximisation for specific policy structures. (b): utility maximisation under $\varepsilon$-statistical parity constraint ($\varepsilon = 0.1$). The results are averaged over $5$ repeats (each with different simulated data). The black lines show the standard deviation.}
\label{fig:college_admissions}
\end{figure*}


\paragraph{Unconstrained optimisation.} 
When solving the pure utility maximisation problem without additional constraints, we observe that the policy gradient algorithm with separable linear policies and the stochastic greedy algorithm achieve the highest expected utility (Figure \ref{fig:unconstrained_results}). Both algorithms yield almost meritocratic fair policies as $\devswap$ and $\devchange$ show small values. In fact, in the unconstrained setting, a deterministic policy that maximises utility is always meritocratic as shown in  Lemma \ref{lemma:determinstic_optimal_meritocractic}. 
The non-zero values of $\devswap$ and $\devchange$ can be explained by the imperfect predictions of the underlying predictive model. {Moreover, the policy gradient algorithm might not have converged to a fully deterministic solution yet, resulting in slightly larger deviations from meritocracy than for the stochastic greedy algorithm, which always yields a deterministic policy.} The policy gradient algorithm with threshold policies achieves much lower expected utility than with separably linear policies, which is expected since the policy space is highly constrained. 
Whereas the threshold policy satisfies local stability, it deviates heavily from swap stability. 
The uniform set selection policy (which is egalitarian) unsurprisingly achieves the lowest expected utility and is swap stable, i.e. $\devswap=0$, as all individuals have equal probability of selection. However, the uniform policy is not locally stable due to $\devchange$ being non-zero.

\paragraph{Optimisation under 
statistical parity.} 
The additional $\varepsilon$-statistical parity constraint 
causes a drop in expected utility for all algorithmic approaches (Figure \ref{fig:constrained_results}). Most notably, we see a rise in deviation from swap stability for separable linear policies, threshold policies, and the stochastic greedy algorithm. This suggests that the statistical parity constraint is enforcing the selection of worse individuals (in terms of merit) over better ones on the basis of their group affiliations. Moreover, we observe high deviation from local stability for the stochastic greedy algorithm. This is in accordance with its notable decrease in utility in the constrained setting.\footnote{We want to emphasise again that the stochastic greedy algorithm is typically deployed for unconstrained utility maximisation settings. As the stochastic greedy algorithm performs successive selection of individuals it may stop selection prematurely in the presence of additional constraints such as statistical parity, which can prevent it from adding an individual to its selection.}  
The uniform selection policy satisfies the $\varepsilon$-statistical parity constraint and we thus observe no changes. 


In general, we observe that unconstrained utility maximisation using separable linear policies or the stochastic greedy selection, yields approximately meritocratic fair decisions. However, once additional constraints on the policy structure, e.g.\ threshold policies, or the selection probabilities, e.g.\ statistical parity, are being imposed, utility maximisation appears to clash with our notion of meritocracy.

\section{Discussion and Future Work}\label{sec:discusion and future work}

We have provided a first look into how one can define meritocracy in a general set selection scenario. 
We define expected marginal contribution (EMC) as the quantifier of an individual's contribution to DM's expected utility. We propose to measure meritocratic fairness with EMC as a dynamic representative of individual merits, in contrast to the static scoring or ranking systems.
We show that the Shapley value is a special case of the EMC, when the DM applies an egalitarian policy, and that it satisfies analogues of the Shapley `fairness' axioms, as well as  two notions of selection stability.


We also observe that any utility maximising deterministic policy satisfies meritocracy (Lemma \ref{lemma:determinstic_optimal_meritocractic}) but it might not be the case for stochastic policies (Lemma \ref{lemma:stochastic_optimal_not_meritocractic}). Specially, the counterexample developed to prove this fact indicates that this might require {\em individually smooth} policies in the sense that similar individuals (in terms of compatibility across optimal sets) are being selected similarly often. Similar types of observations are found in the fairness literature, where smoothness induces individual fairness. It would be interesting in future to explore this link between smoothness and fairness to design meritocratic fair policies.

In particular, our experiments support the connection between meritocracy, EMC and policy gradient for linearly separable policies.
A similar connection for threshold policies seems unlikely given our experimental results and the fact that the policy parameter space has a different dimension from the space of individual contributions. This is perhaps an important consideration for policy makers that create formulaic decision rules. 
In future work, it will be interesting to analyse the case where individual decisions must be made independently of the remaining applicants. 

Our experiments also show that the meritocratic fairness deteriorates as we impose group fairness constraints, such as statistical parity, on the policy. This indicates the philosophical and quantitative tension between meritocratic notions of fairness and demographic notions of fairness \citep{Binns2020}. Meritocratic fairness cares about the individuals and tries to influence the selection policy in terms of the worth (or merit) of individuals. In contrast, demographic fairness cares about different subpopulations in the candidate pool and tries to influence the policy in order to select from these subpopulations equally or proportionally.  
Thus, under some circumstances meritocratic and demographic notions of fairness may be conflicting. 
Our experimental observations resonate this and our approach of formulating statistical parity as a constraint on the policy and EMC as a measure of merit allows us to quantify this tension. 
In future, it would be interesting to examine general conditions on the policy constraints under which demographic notions of fairness align with meritocracy.

Finally, another interesting line of work would be to extend this framework to matching problems.

\balance
\bibliographystyle{ACM-Reference-Format}
\bibliography{ref}




\newpage

\appendix

\allowdisplaybreaks

\section{Proofs}

\subsection{Proofs from \cref{sec:EMC and Meritocracy}}\label{app:proofs from sec 4}
\begin{proof}[Proof of Lemma \ref{lemma:axioms_of_fair_division}]
Symmetry: If $\util(\act+i) = \util(\act+j)$ for all $\act \in \Act$, then the marginal contributions of $i$ and $j$ for all $\act$ must be equal. This implies that $\emc_i(\util, \pol) = \emc_j (\util, \pol)$ for any policy $\pol \in \Pol$. 

Linearity: Using that $(\alpha \util_1 + \beta \util_2)(\act) = \alpha \util_1(\act) + \beta \util_2(\act)$, we get
\begin{eqnarray*}
 &&  \emc_i (\alpha \util_1 + \beta \util_2, \pol)   = \sum_{\act \in \Act} \pol(\act) [ (\alpha \util_1 + \beta \util_2)(\act +i)\\
&&\quad - (\alpha \util_1 + \beta \util_2)(\act )] \\  
&&= \sum_{\act \in \Act} \pol(\act) [ \alpha \util_1(\act+i) - \alpha \util_1(\act) + \beta \util_2(\act +i) - \beta \util_2(\act)] \\
&&= \alpha \emc_i (\util_1, \pol) + \beta \emc_i (\util_2, \pol).  
\end{eqnarray*}

Null Player Property: If $\util(\act + i) = \util(\act)$ for all $\act \in \Act$, then the marginal contributions of individual $i$ are zero for all $\act \in \Act$. Thus, the expected contributions under any policy $\pol \in \Pol$ must be zero as well, i.e.\ $\emc_i(\util, \pol) = 0$. 
\end{proof}

\begin{proof}[Proof of Lemma \ref{lemma:swap-stability equivalence}]
First, recall that $\util(\pol -i - j) = \util(\pol')$, where $\pol' (\act) = \pol(\act) + \pol (\act + i) + \pol(\act + j) + \pol(\act + i +j)$ for $a_i = a_j = 0$, otherwise $\pol(\act) = 0$. Thus, 
\begin{eqnarray*}
&&    \emc_i (\util, \pol - i -j) \\ 
    && = \sum_{\substack{a_i = 0 \\ a_j = 0}} \pol'(\act) [\util(\act + i) - \util(\act)] \\
    && = \sum_{\substack{a_i = 0 \\ a_j = 0}} \big( \pol(\act) + \pol(\act + i) + \pol(\act + j) \\
&&\qquad + \pol(\act + i + j)\big) \util(\act + i) - \util(\pol - i -j) \\
    && =\sum_{\substack{a_i = 1 \\ a_j = 0}} \big( \pol(\act-i) + \pol(\act) + \pol(\act - i + j)\\
&&\qquad + \pol(\act + j)\big) \util(\act) - \util(\pol - i - j) \\
    && = \util(\pol +i -j ) - \util(\pol - i -j)\\
&& = \sum_{\act \in \Act} \pol(\act) [ \util(\act + i - j) - \util(\act - i - j)]
\end{eqnarray*}
Now, it follows that 
\begin{align*}
    &\util (\pol + i -j ) - \util(\pol - i + j) \\
    & = \util(\pol + i - j) - \util(\pol - i -j) + \util (\pol - i -j ) - \util(\pol - i + j) \\
    & = \emc_i (\util, \pol - i -j) - \emc_j (\util, \pol - i -j).
\end{align*}
This shows that $\util(\pol+i-j) \geq \util (\pol-i+j)$ is equivalent to $\emc_i (\util, \pol -i -j ) \geq \emc_j (\util, \pol-i-j)$.
\end{proof}

\begin{proof}[Proof of Lemma \ref{lemma:local-change-stability equivalence}]
This readily follows from the definition of the EMC of individual $i$ under $\pol$ as $\emc_i (\util ,\pol) = \sum_{\act \in \Act} \pol(\act) [\util(\act + i) - \util(\act)] =  \util(\pol + i) - \util (\pol)$.
\end{proof}

\begin{proof}[Proof of Lemma \ref{lemma:determinstic_optimal_meritocractic}]
Let 
$\Act^* \defn \{ \act \in \Act \colon \util(\act) = \max_{\act' \in \Act} \util(\act') \}$
denote the set of utility maximising selections.  
Note that for any optimal (utility maximising) policy $\pol$ with $\pol(\act) > 0$, it must hold that $\act \in \Act^*$. 
In particular, if $\pi$ is a deterministic optimal policy, then $\pi(\act^*) = 1$ for some $\act^* \in \Act^*$. Let us now show that such a deterministic optimal policy $\pi$ that selects $\act^* \in \Act^*$ with probability one is swap stable. For all $i,j \in \pop$ with $a^*_i = a^*_j$ we have nothing to check. Hence, consider $i,j \in \pop$ such that $\act^*_i = 1$ and $\act^*_j = 0$. We then find that 
\begin{eqnarray*}
&&    \util (\pi + i -j) - \util(\pi -i +j) = \util(\act^* + i -j )\\
&&\qquad - \util (\act^* -i +j) = \util(\act^*) - \util(\act^* - i +j ) \geq 0,
\end{eqnarray*}
by optimality of set $\act^*$. We have thus shown that $\pol$ is swap stable. To show that $\pol$ is locally stable, note that if $\pol$ is any (possibly stochastic) optimal policy, we have that $\emc_i (\util, \pol) \leq 0$ for all $i \in \pop$. This follows directly from the fact that the support of $\pol$ must be a subset of $\Act^*$ and $\util(\act^* + i) - \util (\act^*) \leq 0$ for all $i \in \pop$. By merit of \cref{lemma:local-change-stability equivalence}, we have therefore shown that any optimal policy $\pol$ is locally stable, including all stochastic optimal policies. 
\end{proof}

\begin{proof}[Proof of Lemma \ref{lemma:stochastic_optimal_not_meritocractic}]
By counterexample. Let 
$\Act^* \defn \{ \act \in \Act \colon \util(\act) = \max_{\act' \in \Act} \util(\act') \}$
denote the set of utility maximising selections. For $i, j \in \pop$, let there be $\act \in \Act^*$ with $\act_i = 1$, $\act_j = 0$, and $ \mathbf{b} \in \Act^*$ with $\mathbf{b}_i = 0$, $\mathbf{b}_j = 1$. Since $\act, \mathbf{b} \in \Act^*$, the policy \begin{align*}
    \pol(\act) = {}^{2}/{}_{3}, \quad \pol(\bct)= {}^{1}/{}_{3} = 1- \pol(\act)
\end{align*}
is optimal. In particular, note that the probability of selecting $i$ under $\pol$ is strictly greater than the probability of selecting $j$. Thus, for swap stability to hold, $\pol$ must satisfy $\util(\pol + i -j) \geq \util(\pol - i + j)$. Now, let us assume that $\util(\act-i+j) = \util(\act)$ and $\util(\bct+i-j) < \util(\bct)$, which implies
\begin{align}\label{equation:stochastic_not_meritocractic_1}
2(\util (\act) - \util(\act - i + j)) < \util(\bct) - \util(\bct+i-j).
\end{align}
This can be understood as saying that $j$ is more ``compatible'' than $i$, since $j$ achieves maximal utility when added to $\act$, but adding $i$ to $\bct$ yields sub-optimal utility.
Now, by equation \eqref{equation:stochastic_not_meritocractic_1}, we have that
\begin{eqnarray*}
&&    \util (\pi + i -j ) - \util (\pi -i +j) = \frac{2}{3} (\util(\act) - \util(\act-i +j)) \\
&&\qquad + \frac{1}{3} ( \util (\bct +i -j) - \util (\bct) ) < 0.
\end{eqnarray*}
Thus, $\pi$ is not swap stable. 
\end{proof}

\subsection{Proofs from \cref{sec:util-maxim-marg}}\label{app:proof_section_5}

\begin{proof}[Proof of Lemma \ref{lemma:softmax no. 2}]
Define $Z(\vparam) = \sum_{\act \in \Act} e^{\beta \vparam^\top \act}$ and note that 
\begin{align*}
    \frac{\partial}{\partial \param_i} \log Z(\vparam) = \frac{1}{Z(\vparam)} \sum_{\act \in \Act \colon a_i = 1} e^{\beta \vparam^\top \act} = \beta \sum_{\act \in \Act \colon a_i =1 } \pol_\vparam (\act),
\end{align*}
Recall our convention that $\pol_\vparam (a_i = x) \defn \sum_{\act \in \Act \colon a_i = x } \pol_\vparam (\act)$ for $x \in \{0,1\}$. In the following, we omit the factor $\beta$ as it will be nothing but a constant factor. 
\begin{eqnarray*}
&&    \frac{\partial}{\partial \param_i} \util (\pol_\vparam)  \\
&&= \sum_{\act \in \Act} \frac{\partial}{\partial \param_i} \pol_\vparam (\act) \util (\act) \\ 
&& = \sum_{\act \in \Act} \pol_\vparam (\act) \frac{\partial}{\partial \param_i} \log( \pol_\vparam (\act)) \util (\act) \\
 && = \sum_{\act \in \Act} \pol_\vparam (\act) \frac{\partial}{\partial \param_i} (\vparam^\top \act - \log Z(\vparam)) \util (\act) \\
   && = \sum_{\act \in \Act} \pol_\vparam (\act) \big( a_i - \sum_{\act^\prime \in \Act \colon a^\prime_i = 1} \pol (\act^\prime) \big) \util (\act) \\
    && = \sum_{\act \in \Act \colon a_i =1} \pol_\vparam (\act) \util (\act) - \pol_\vparam (a_i =1) \util(\pol_\vparam) \\
    && = (1-\pol_\vparam ( a_i =1) )\\
&&  \sum_{\act \in \Act \colon a_i = 1} \pol_\vparam (\act) \util (\act) - \pol_\vparam (a_i =1) \\
&&\sum_{\act \in \Act \colon a_i = 0} \pol_\vparam (\act) \util (\act) 
     = \pol_\vparam (a_i =0) \, e^{\beta \param_i} \\
&& \sum_{\act \in \Act \colon a_i = 0} \pol_\vparam (\act) \util (\act + i) - \pol_\vparam (a_i =  1) \\
&&\sum_{\act \in \Act \colon a_i = 0} \pol_\vparam (\act) \util (\act) 
     = \pol_\vparam ( a_i =1)\\
&& \sum_{\act \in \Act \colon a_i = 0} \pol_\vparam (\act) \util (\act + i) -  \pol_\vparam (a_i =  1) \\
&&\sum_{\act \in \Act \colon a_i = 0} \pol_\vparam (\act) \util (\act) 
     = \pol_\vparam (a_i =1)\\
&& \sum_{\act \in \Act \colon a_i = 0} \pol_\vparam (\act)  [\util (\act + i) - \util (\act)] 
     = \pol_\vparam (a_i =1)  \\
&&\sum_{\act \in \Act} \pol_\vparam (\act)[\util (\act + i) - \util (\act)] \\
&&\qquad {\textit{(for $\act$ with $\act_i = 1$ the summand is zero)}}\\ 
    && = \pol_\vparam (a_i =1) \,  \emc_i (\util , \pol_\vparam).
\end{eqnarray*}
\end{proof}

\begin{proof}[Proof of Lemma \ref{lemma:multinomial gradient}]
If $\pol_\vparam(a_i = 0) > 0$, we obtain
\begin{eqnarray*}
&&    \frac{\partial}{\partial \param_i} \util (\pol_\vparam)  = \frac{\partial}{\partial \param_i} \sum_{\act \in \Act} \pol_\vparam (\act) \util (\act) \\
    && = \sum_{\act \in \Act} \frac{\partial}{\partial \param_i} \pol_{\param_i} (a_i) \prod_{j \neq i} \pol_{\param_j} (a_j) \util(\act) \\
    && = \sum_{\act \in \Act} (\ind{a_i = 1} - \ind{a_i = 0} ) \prod_{j \neq i} \pol_{\param_j} (a_j ) \util(\act ) \\
    && = \sum_{\act \in \Act \colon a_i = 0} \prod_{j \neq i} \pol_{\param_i} (a_i)  \util (\act + i)\\
&&  - \sum_{\act \in \Act \colon a_i = 0}  \prod_{j \neq i} \pol_{\param_i} (a_i)  \util (\act) \\
    && = \sum_{\act \in \Act \colon a_i = 0}  \prod_{j \neq i} \pol_{\param_j} (a_j)  [ \util (\act + i ) - \util (\act ) ]  \\
    && = \frac{1}{\pol_{\param_i} (a_i = 0)} \sum_{\act \in \Act \colon a_i = 0} \pol(\act ) [ \util (\act + i) - \util (\act ) ]  \\
    && = \frac{1}{\pol_{\param_i} (a_i = 0)} \sum_{\act \in \Act} \pol(\act ) [ \util (\act + i) - \util (\act) ]\\
&& \qquad {\textit{(for $\act$ with $\act_i = 1$ the marginal contribution is zero)}}\\ 
    && = \frac{\emc_i (\util ,\pol_\vparam )}{\pol_{\param_i} (a_i = 0 )}. 
\end{eqnarray*}
In line four of our calculations, we needed the following technical identity
\begin{align*}
    \sum_{\act \in \Act} \ind{a_i = 1} \prod_{j \neq i} \pol_{\param_j} (\act_j) \util (\act) = \sum_{\act \in \Act \colon \act_i = 0} \prod_{j \neq i} \pol_{\param_j} (\act_j) \util (\act + i).
\end{align*}
For completeness, we rigorously prove it here. The left hand side is equal to 
\begin{align*}
    \sum_{\act \in \Act \colon \act_i = 1} \prod_{j \neq i } \pol_{\param_j} (\act_j ) \util (\act) = \sum_{\act \in \Act \colon \act_j = 1} \prod_{j \neq i} \pol_{\param_j} (\act_j) \util (\act+ i).
\end{align*}
Now, let $\act$ and $\act^\prime$ only differ in the $i$-th element, namely, $\act_i = 1$ and $\act^\prime_i = 0$. Clearly,
\begin{align*}
    \prod_{j \neq i} \pol_{\param_j} (a_j) = \prod_{j \neq i} \pol_{\param_j} (a^\prime_j) \quad  \text{ and }  \quad \util(\act + i) = \util (\act^\prime + i),
\end{align*}
which yields
\begin{align*}
    \sum_{\act \in \Act \colon a_i = 1} \prod_{j \neq i} \pol_{\param_j} (a_j) \util (\act+ i) = \sum_{\act \in \Act \colon a_i = 0} \prod_{j \neq i} \pol_{\param_j} (a_j) \util (\act + i).
\end{align*}
\end{proof}

\section{Details of the Experimental Setup}\label{app:details_of_experiments}

\subsection{Details of the Policy Gradient Algorithm}
For the constrained optimisation problem (under statistical parity), we maximise the penalised utility: 
\begin{eqnarray*}
&&    U'(\pi) = U(\pi) - \lambda f(\pi) \text{ with } f(\pi) = (\pol(\act_i = 1 \mid i \text{ is male})\\
 && - \pol(\act_j = 1 \mid j \text{ is female}))^2. 
\end{eqnarray*}
This leads to the constraint policy gradient algorithm.

\begin{algorithm}[H]
	\caption{Constrained policy gradient algorithm}
	\label{alg:constrained_pol_opt}
	\begin{algorithmic}[1]
		\State Input: A population $\pop$ with features $\bx$, a utility function $\ut$, a constraint function $f$ and $\varepsilon > 0$.
		\State Initialise $\vparam^0$, threshold $\delta > 0$, learning rate $\eta^0 > 0$ and $\lambda^0\leq 0$
		\While{$\| \vparam^{i+1} - \vparam^i \| > \delta$}
		\State Evaluate $\nabla_{\theta} (\util (\pol_\vparam , \bx) -\lambda^i (f(\pol_\vparam,\bx) - \varepsilon))$ and $\nabla_{\lambda} (\util (\pol_\vparam , \bx) -\lambda (f(\pol_\vparam,\bx) - \varepsilon))$ using $\bx$
		\State $\vparam^{i+1} \leftarrow \vparam^i + \eta^i \left\{\nabla_{\theta} \left[\util (\pol_\vparam , \bx) -\lambda (f(\pol_\vparam,\bx) - \varepsilon) \right]_{\vparam=\vparam^{i}}\right\}$
		\State $\lambda^{i+1} \leftarrow \lambda^i + \eta^i \left\{\nabla_{\lambda} \left[\util (\pol_\vparam , \bx) -\lambda (f(\pol_\vparam,\bx) - \varepsilon) \right]_{\lambda = \lambda^i} \right\}$
        \State $i+\!+$
		\EndWhile
		\State \textbf{return} $\pi_{\vparam^{i+1}}$
	\end{algorithmic}
\end{algorithm}
In the algorithm above, $\eta^i$ is updated according to a schedule so that $\eta^{i+1} = 0.9 \eta^i$ whenever convergence stalls.  

\paragraph{Utility gradient.} Note that for all policies we consider $\pol(\act \mid \bx) = \prod_{i \in \pop} \pol(\act_i \mid \bx)$. For the separable linear policies, the gradient $\nabla_{\vparam} \util(\pol_\vparam, \bx)$ is given by Lemma \ref{lemma:multinomial gradient}. 
Then, to estimate $\emc_i (\util, \pol_\vparam, \bx)$, we sample $n=40$ sets $a_1, \dots, a_n \in \Act$ from $\pol_\vparam (\act \mid \bx)$ and approximate the EMC of individual $i$ given $\pol_\vparam$ and $\bx$ by 
\begin{align*}
    \emc_i (\util, \pol_\vparam, \bx) \approx n^{-1} \sum_{k=1}^n [\util(\act_k +i , \bx) - \util(\act_k, \bx)].
\end{align*}
Note that due to the separable structure of $\pol_\vparam$, we have $\pol_\vparam (\act_i = 1 \mid \bx) = \sum_{\act \in \Act \colon \act_i = 1} \pol_{\vparam} (\act \mid \bx) = \param_i$ and are thus not required to approximate this value. Moreover, recall that, using our predictive model, we obtain the expected utility of selecting a set $\act$ given $\bx$ by marginalising over the outcomes: $\util(\act, \bx) = \sum_{\by \in \CY} \Pr (\by \mid \act, \bx) u(\act, \by)$.
For the logit threshold polices with $\vparam = (\param_1, \dots, \param_{|\mathcal{X|}})$, the gradient can be computed as
\begin{eqnarray*}
    && \frac{\partial}{\partial \param_i} \util(\pol_\vparam, \bx) \\
    && = \sum_{\act \in \Act} \frac{\partial}{\partial \param_i} \pol_\vparam(\act \mid \bx) \util(\act, \bx) \\
    && = \sum_{\act \in \Act}  \pol_\vparam(\act \mid \bx) \util(\act , \bx)  \frac{\partial}{\partial \param_i} \log (\pol_\vparam(\act \mid \bx) ) \\
    && = \sum_{\act \in \Act} \pol_\vparam(\act \mid \bx) \util(\act, \bx) \sum_{j = 1}^N \frac{\partial}{\partial \param_i} \log(\pol_\vparam (\act_j \mid \bx_j)) \\
    && = \sum_{\act \in \Act}  \pol_\vparam(\act \mid \bx) \util(\act, \bx) \sum_{j=1}^N \bx_{j, i} \Big( \ind{\act_j = 1}\\
&&\qquad \pol_\vparam(\act_j = 0 \mid \bx_j) - \ind{\act_j = 0} \pol_\vparam(\act_j = 1 \mid \bx_j) \Big).
\end{eqnarray*}
In our experiments, we again sample $n=40$ sets $\act_1, \dots, \act_n \in \Act$ from $\pol_\vparam (\act \mid \bx)$, and then approximate the gradient by 
\begin{eqnarray*}
&& \hspace*{-11pt}   \frac{\partial}{\partial \param_i} \util(\pol_\vparam, \bx)  \approx n^{-1} \sum_{k=1}^n \util(\act_k, \bx) \sum_{j=1}^N \bx_{j,i}\\
&&\hspace*{-11pt}\qquad \left(  \ind{\act_{k,j} = 1} \pol_\vparam (\act_j = 0 \mid \bx_j) - \ind{\act_{k, j} = 0} \pol_\vparam (\act_j = 1 \mid \bx_j) \right).
\end{eqnarray*}
Note that we can compute $\pol_\vparam(\act_j = 0 \mid \bx_j)$ directly as $\pol_\vparam(\act_j = 0 \mid \bx_j) = (1+e^{\vparam^\top \bx_j})^{-1}$.

\paragraph{Constraint gradient.} We can rewrite the penalty term $f$ using the affiliation vectors $M,F \in \{0,1\}^N$ with $M_i = 1$ if individual $i$ is male and $F_i = 1$ if individual $i$ is female. Then, $\lVert M \rVert_1$ 
denotes the number of males and $\lVert F \rVert_1$ 
denotes the number of females ($\lVert M \rVert_1 + \lVert F \rVert_1 = \pop$). We get the following expression for the penalty term $f(\pol)$:
\begin{align*}
    f(\pol) = \left( \frac{\sum_{i \colon  M_i = 1}\pi(a_i = 1\mid \bx)}{\lVert M \rVert_1 } - \frac{\sum_{i \colon F_i = 1}\pi(a_i = 1 \mid \bx)}{\lVert F \rVert_1 } \right)^2 - \varepsilon
\end{align*}
Its gradient w.r.t.\ $\pol$ is then given by
\begin{eqnarray*}
&&    \nabla_{\pol} f(\pol) = 2\left( \frac{\sum_{i \colon  M_i = 1}\pi(a_i = 1 \mid \bx)}{\lVert M \rVert_1 } - \frac{\sum_{i \colon F_i = 1}\pi(a_i = 1 \mid \bx)}{\lVert F \rVert_1 }  \right) \\
&&\qquad \left( \frac{M}{\lVert M \rVert_1} - \frac{F}{\lVert F \rVert_1} \right).
\end{eqnarray*}

\paragraph{Policy gradient.}
For linear separable policies, the gradient of $\pol$ w.r.t.\ $\vparam = (\param_1, \dots, \param_N)$ is given by a diagonal $N\times N$-matrix with entries $( \ind{a_i = 1} - \ind{a_i = 0})_{i \in [N]}$.
For the logit threshold policy with $\vparam = (\param_1, \dots, \param_m)$ ($m = |\mathcal{X}|$), we have 
\begin{eqnarray*}
&&    \nabla_\vparam \pol_\vparam = \big[\bx_{i,j} \big(\ind{a_i = 1} e^{-\vparam^\top \bx_i} (1+ e^{-\vparam^\top \bx_i})^{-2}\\
&&\qquad - \ind{a_i = 0} e^{\vparam^\top \bx_i} (1+ e^{\vparam^\top \bx_i})^{-2} \big)   \big]_{i \in \pop}^{j \in [m]}.
\end{eqnarray*}
As before, we then sample from $\pol(\act \mid \bx)$ to approximate the gradient $\nabla_{\vparam} \pol_\vparam$ and thereby obtain $\nabla_\vparam f(\pol) = \nabla_{\vparam} \pol_\vparam \nabla_{\pol} f(\pol)$.



\subsection{Simulator Description}\label{app:simulator}
\label{subsec:simul}

In the following, we describe the simulator that was used to generate the data for our experiments. The simulator generates a population of applicants in two steps: 1) generating applicants' features and 2) generating outcomes - grades in three courses given the applicant features (assuming admission). The dataset itself will not be published as it is not publicly open. 

\paragraph{Generating students.}
We wish to generate a population of applicants for a specific study program. To this end, we use ctgan\footnote{\url{https://github.com/sdv-dev/CTGAN}}, a deep learning based synthetic data generator for tabular data, that can learn from real data and generate synthetic clones with high fidelity.
We provide this generator with a clean version of the application data table (described in \ref{sec:experiments}). The data is cleaned by the following steps: First, we keep applications for the selected study program. Then, we filter out all applications that are not through normal admission or do not have a valid application, resulting in roughly 40,000 applications.
In some cases, attributes of an applicant might be missing in the application for one program but are present in an application for a different program. We fill missing values of GPA, science points and language points for applicants using the matching values in different study programs in the same year (if exist). After this step, we select the following features: application year, gender, country code for citizenship, country code for educational background, priority, GPA, science points, language points and other points. The age at time of application is also added (calculated as the difference between application year and year of birth). Finally, all applications with remaining missing values are removed, resulting in approximately 30,000 applications. The ctgan generator is trained using this data over $10$ epochs with default parameters (see Figure~\ref{fig:simulated histograms} for evaluation).

\begin{figure*}[t]
\centering
\includegraphics[scale=0.4]{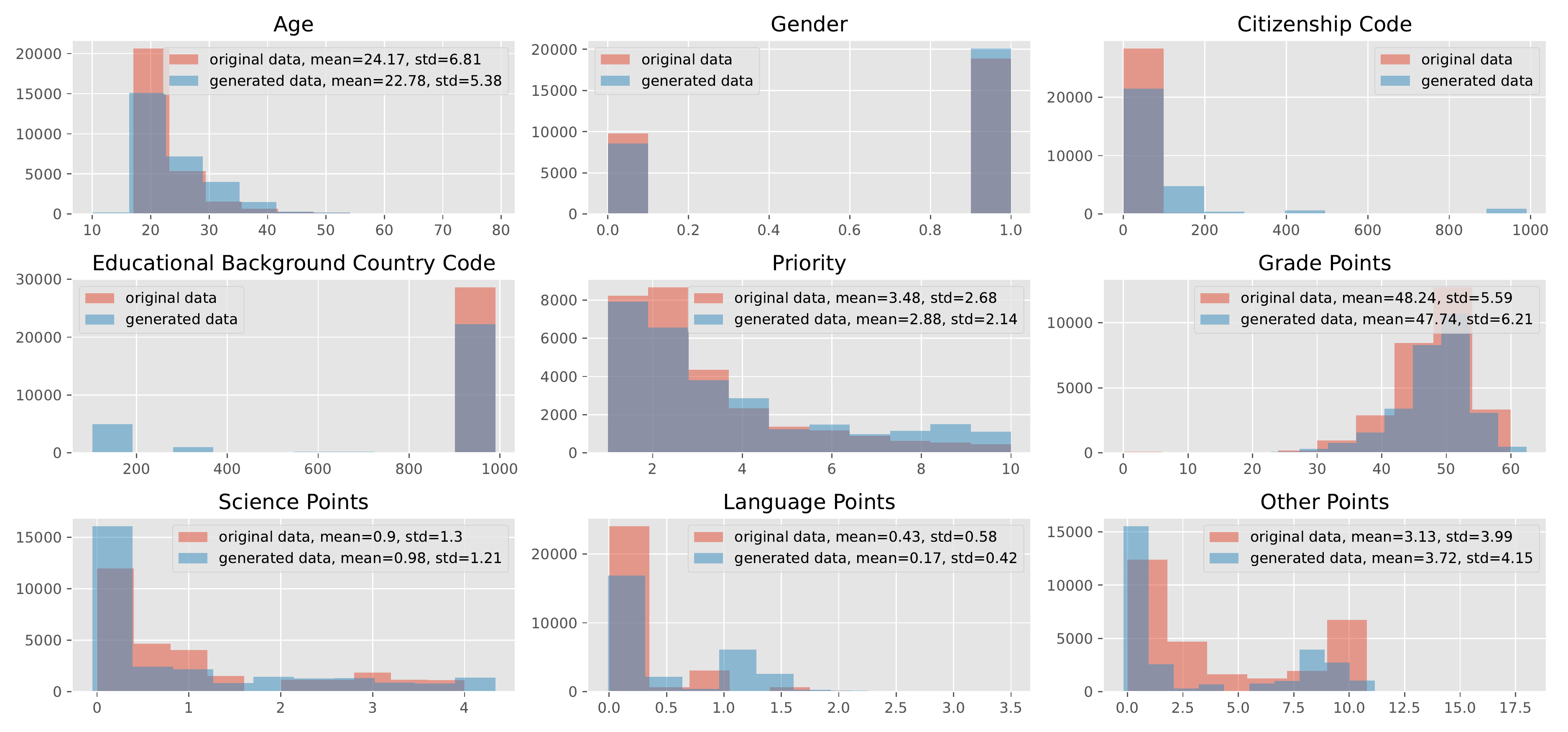}
\caption{Histograms of the different features for the original data (training set) and for the samples generated by CTGAN, using sample size equal to the size of the training set.}
\label{fig:simulated histograms}
\end{figure*}


\paragraph{Generating outcomes.}
In addition to applicant features, we also generate outcomes for each candidate - grades in three mandatory courses in the study program (one first year course and two second year courses). Naturally, the data only provides grades for admitted students. Thus, we use existing data to train a linear regression model and use it to produce outcomes for generated applicants.

We construct the training data from the two data tables described in \ref{sec:experiments}. The clean version of the application data is used as training features, with the additional following steps: For each applicant, only the last application to this program is kept, the rest are dropped from the table. This results in $~15,000$ applicants. We wish to consider the set of accepted  students who took the first year course, had data of at least two years in the program, had the right to study and the right to take the exam. This filter results in $~900$ students. Then, we also filter out students who did not have a grade in the A-F scale in the first year course. We do include students with the mark 'X' stating they did not show up for the exam, and consider it as 'F'. The final set includes $557$ students. 
The training targets are the last grades of each student in the final set in the three selected courses.
Before training, all numeric features are scaled between $0$ and $1$, categorical features are one-hot encoded, and the targets are translated to numeric values in $[0,1]$ as follows: 'A': 1, 'B':0.8 , 'C':0.6 , 'D':0.4 , 'E':0.2 , 'F':0 . Missing values (which may exist for the second year courses, for example in cases of dropouts) are considered as $0$.
A linear regression is fitted to the data using default parameters of scikit-learn. Predictions made by the model are clipped to $[0,1]$, resulting in r2 score of $0.05$.

\paragraph{Modeling past admissions.} To generate the data of students from prior years, we use a model representing the admission policy of the study program that was used in earlier years. This admission policy is modeled by a logistic regression with default parameters trained on approximately $30,000$ samples with $~7\%$ positive samples. The training data is identical to the the training data used for generating students, except for the labels which now indicate whether the candidate was offered and accepted the offer to study in the program. Before training, all numeric features are scaled between $0$ and $1$, categorical features are one-hot encoded.
In Norway, a stable marriage algorithm (from the student's perspective) is used to match students to study programs. Thus, a model that can only observe data for one program cannot be completely accurate. Yet, since the features include priority, we can still provide a reasonable prediction. The accuracy achieved by the model given a train-test split of $75:25$ is $93\%$. In the admission process we simulate, the $k$ best scores (probability of being admitted) are being selected.

\subsection{Additional Results}
In addition to Figure \ref{fig:college_admissions}, we include results comprising the performance of the algorithms with respect to the actual outcomes as well as the predicted outcomes in Figure \ref{fig:additional_results}. As shown, the results are almost identical due to the high accuracy ($97\%$) of the regression model. 
These results also include the performance of the historical policy (described in \ref{subsec:simul} under "Modeling past admissions". For these experiments, the policy has selected the top $k$ candidates, where $k$ is set to be identical to the set size selected by the greedy policy. The historical policy does not maximize the utility described in our experiments, thus it seems that it manages to achieve a reasonable utility due to the predetermined set size $k$. This policy has large values of deviation from meritocracy, which is expected because it does not select necessarily the best candidates, but also takes into account their priorities.
In the constrained case, we see a significant increase of $\devchange$, along with an increase of $\devswap$ and decreased utility.


\begin{figure*}[h]
\centering
\begin{subfigure}{.4\textwidth}
  \centering
  \includegraphics[width=\textwidth]{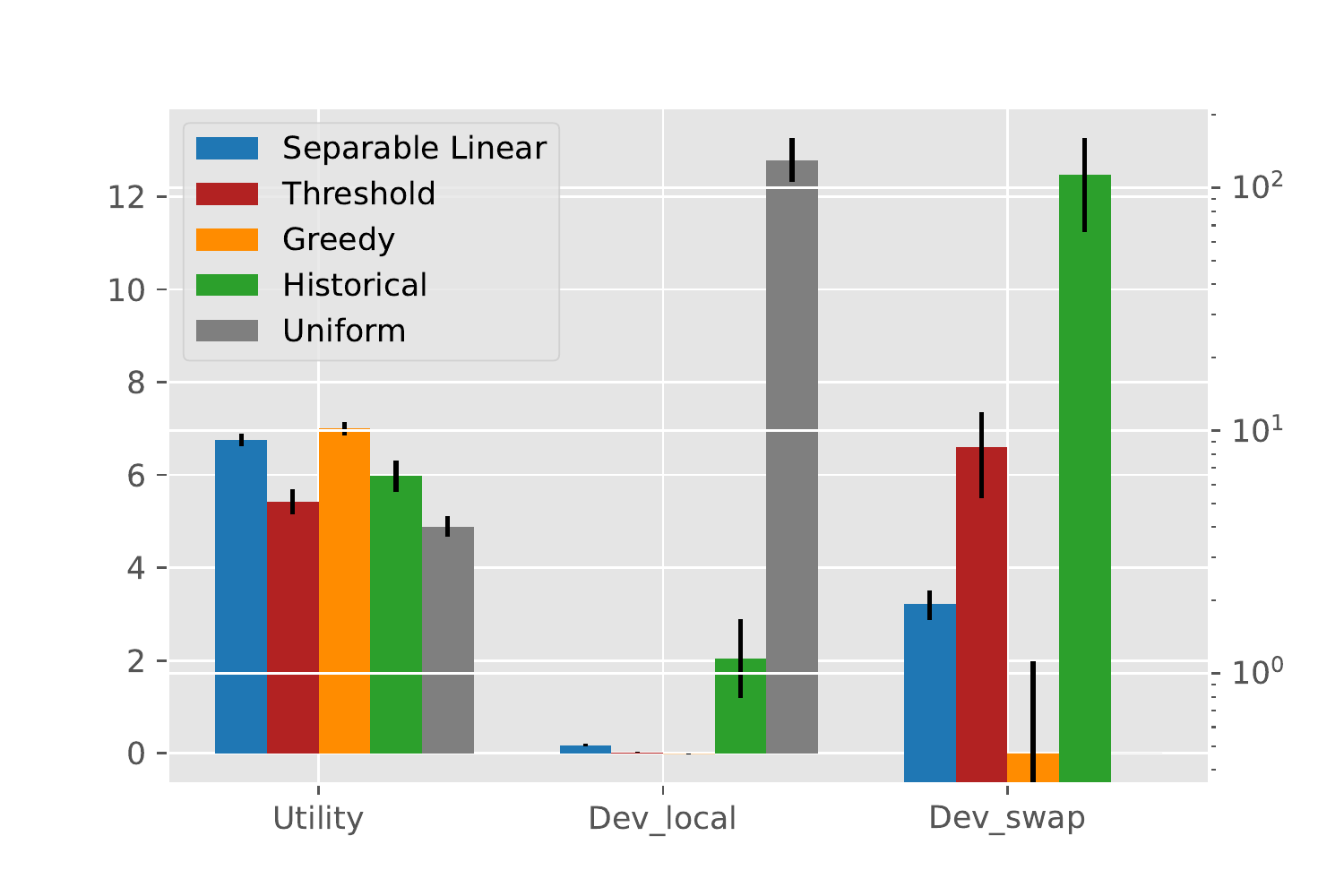}
  \caption{Unconstrained Optimisation (True Outcomes)}
  \label{fig:unconstrained_true}
\end{subfigure}
\begin{subfigure}{.4\textwidth}
  \centering
  \includegraphics[width=\textwidth]{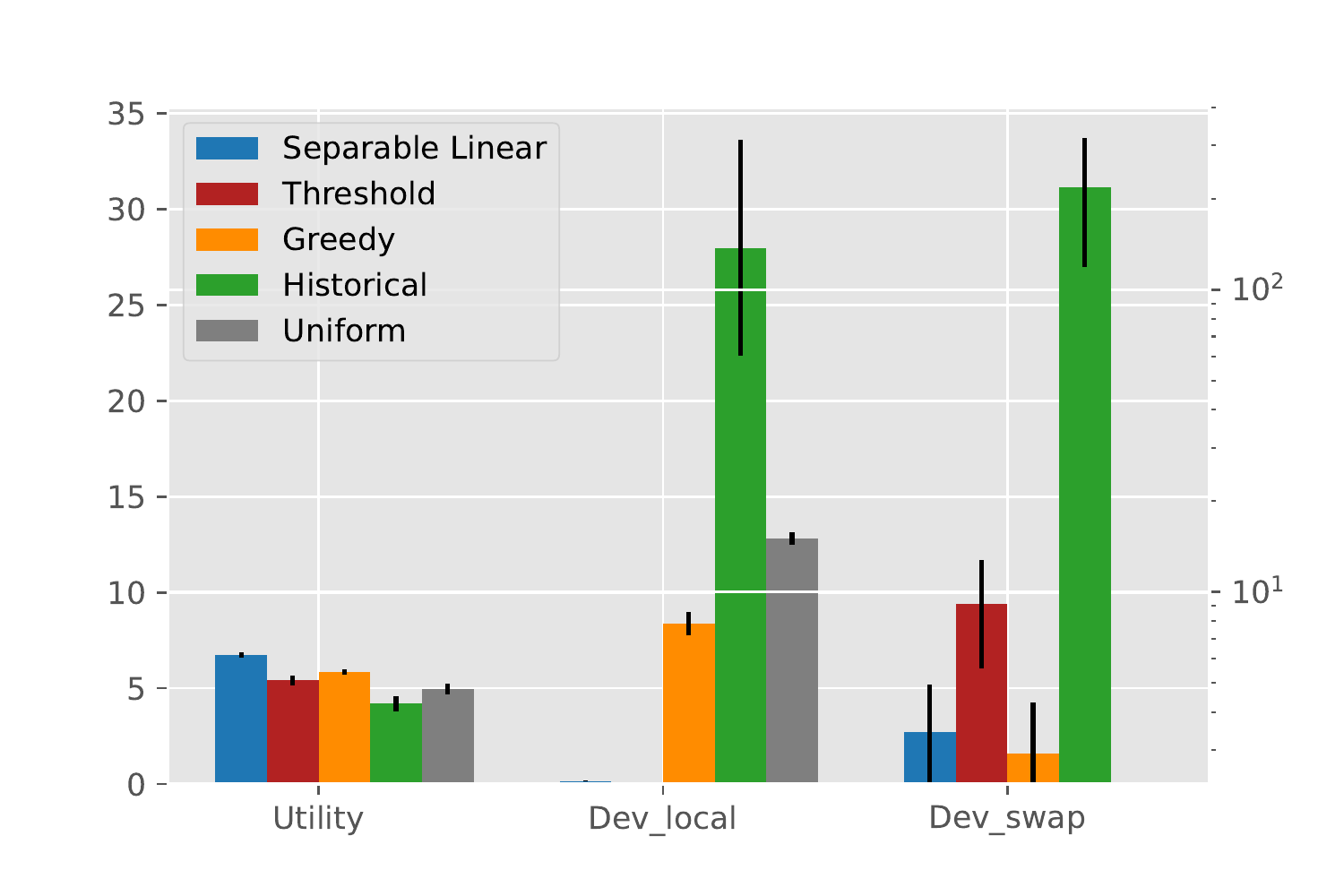}
  \caption{Constrained Optimisation (True Outcomes)}
  \label{fig:unconstrained_estimated}
\end{subfigure}
\begin{subfigure}{.4\textwidth}
  \centering
  \includegraphics[width=\textwidth]{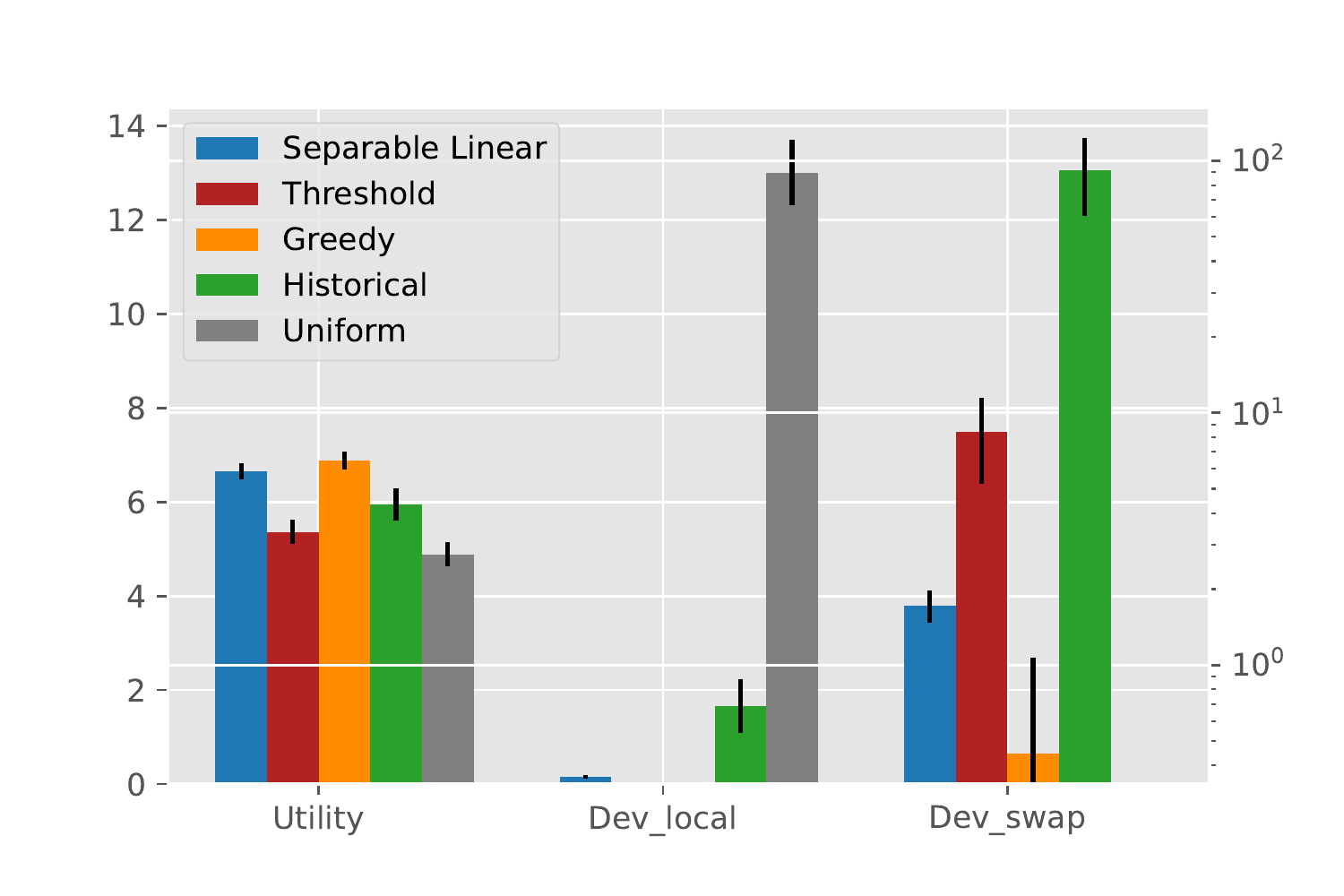}
  \caption{Unconstrained Optimisation (Predicted Outcomes)}
  \label{fig:constrained_true}
\end{subfigure}
\begin{subfigure}{.4\textwidth}
  \centering
  \includegraphics[width=\textwidth]{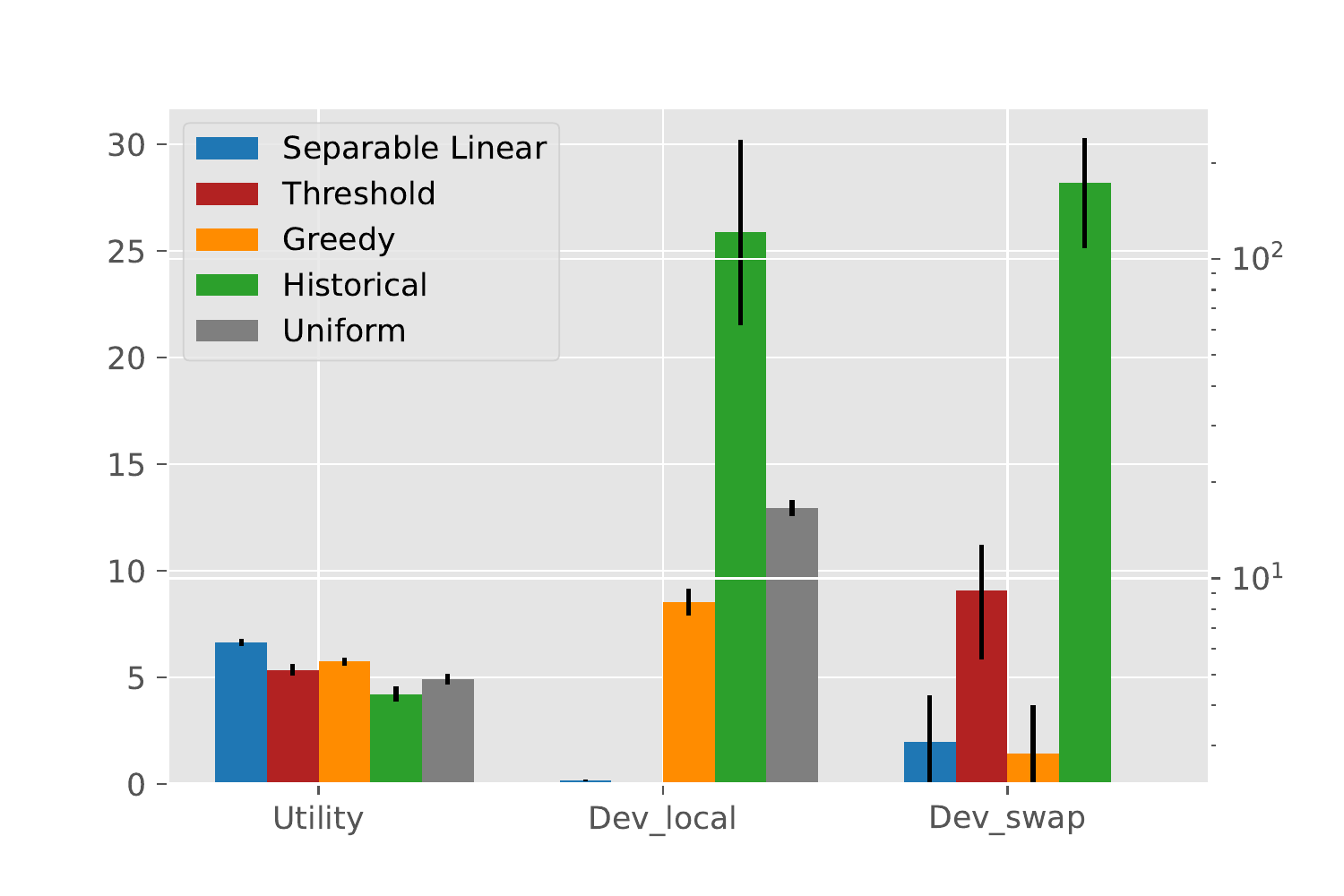}
  \caption{Constrained Optimisation (Predicted Outcomes)}
  \label{fig:constrained_estimated}
\end{subfigure}
\caption{Expected utility, $\devswap$, and $\devchange$ w.r.t.\ the true outcomes and the predicted outcomes of individuals for log-linear utility with $200$ applicants, selection cost $c=0.05$, and bias $\varepsilon = 0.1$. The results are averaged over $5$ repeats (each with different simulated data). The black lines represent the standard deviation. $\devswap$ is presented in log scale (appears on the right)}
\label{fig:additional_results}
\end{figure*}

\end{document}